\documentclass[11pt,a4paper,final]{article}
\pdfoutput=1
\usepackage[utf8]{inputenc}
\usepackage{amsmath}
\usepackage{graphicx}
\usepackage{xcolor}
\usepackage[left=1in,right=1in,top=1in,bottom=1in]{geometry}
\usepackage[breakable]{tcolorbox}

\usepackage{amsthm}

\usepackage{XCharter}
\usepackage[xcharter,vvarbb]{newtxmath}

\usepackage{booktabs} 
\usepackage[ruled, noend]{algorithm2e}
\usepackage{tikz}
\usetikzlibrary{arrows}
\usetikzlibrary{shapes.multipart}
\usepackage{caption}

\SetAlFnt{\small}
\SetAlCapFnt{\small}
\SetAlCapNameFnt{\small}
\SetAlCapHSkip{0pt}
\IncMargin{-\parindent}

\usepackage[numbers,sort&compress]{natbib}

\usepackage[T1]{fontenc}    
\usepackage{hyperref}       
\newcommand{\declarecolor}[2]{\definecolor{#1}{RGB}{#2}\expandafter\newcommand\csname #1\endcsname[1]{\textcolor{#1}{##1}}}

\newcommand{\cH}{\mathcal{H}}

\declarecolor{White}{255, 255, 255}
\declarecolor{Black}{0, 0, 0}
\declarecolor{Maroon}{128, 0, 0}
\declarecolor{Coral}{255, 127, 80}
\declarecolor{Red}{182, 21, 21}
\declarecolor{LimeGreen}{50, 205, 50}
\declarecolor{DarkGreen}{0, 90, 0}
\declarecolor{Navy}{0, 0, 128}

\definecolor{plotblue}{HTML}{377eb8}
\definecolor{plotorange}{HTML}{ff7f00}
\definecolor{plotgreen}{HTML}{4daf4a}

\hypersetup{
    colorlinks=true,
    citecolor=DarkGreen,
    linkcolor=Navy,
    filecolor=magenta,      
    urlcolor=black,
    pdfpagemode=FullScreen,
}
\usepackage{url}            
\usepackage{booktabs}       
\usepackage{nicefrac}       
\usepackage{microtype}      
\usepackage{mathtools}
\usepackage{bbm}
\usepackage{tikz} 
\usepackage{wrapfig}
\usepackage{floatrow}
\usepackage{enumitem}
\usepackage{authblk}
\usepackage[capitalise,noabbrev]{cleveref}
\usepackage{booktabs}
\usepackage{parskip}

\usepackage{color-edits}
\addauthor{kh}{red}
\addauthor{iws}{red}
\addauthor{pa}{cyan}

\usepackage{caption}
\usepackage{subcaption}
\usepackage{thm-restate}

\newtheoremstyle{spaced}
  {0.5em}          
  {0em}          
  {\normalfont}  
  {}             
  {\bfseries}    
  {.}            
  { }            
  {}             

\theoremstyle{spaced}

\usepackage{titlesec}

\titlespacing*{\paragraph}
{0pt}      
{1.0ex}    
{0.8em}    

\newtheorem{theorem}{Theorem}[section]
\newtheorem{lemma}[theorem]{Lemma}

\newtheorem{corollary}[theorem]{Corollary}

\newtheorem{assumption}[theorem]{Assumption}

\newcommand{\cX}{\mathcal{X}}
\newcommand{\cY}{\mathcal{Y}}
\newcommand{\E}{\mathbb{E}}
\newcommand{\kl}{\mathrm{KL}}
\newcommand{\rad}{\mathrm{rad}}
\newcommand{\src}{\mathrm{src}}
\newcommand{\cD}{\mathcal{D}}
\newcommand{\rf}{\mathrm{ref}}
\newcommand{\lo}{\mathrm{low}}
\newcommand{\hi}{\mathrm{high}}
\newcommand{\llama}{\texttt{Llama-3.2-1B}\xspace}
\newcommand{\qwen}{\texttt{Qwen3-8B}\xspace}
\newcommand{\gptfivemini}{\texttt{GPT-5-mini}\xspace}
\newcommand{\gptfivenano}{\texttt{GPT-5-nano}\xspace}

\newcommand{\NN}{\mathbb N}
\newcommand{\PP}{\mathbb P}

\DeclareMathOperator*{\argmax}{argmax}

\begin{document}

\title{\textbf{Post-Training at the Edge of Detectability:\\ A Game-Theoretic Approach to Fine-Tuning}}
\author[1]{Keegan Harris}
\author[1]{Brian W. Lee}
\author[1]{Ian Waudby-Smith}
\author[1]{\\Philip Amortila}
\author[1]{Nika Haghtalab}
\author[1,2]{Michael I. Jordan}

\affil[1]{University of California, Berkeley}
\affil[2]{Inria \& \'Ecole Normale Sup\'erieure}
\affil[ ]{\small\texttt{\{keegan.harris,bl.ee,ianws,p.amortila,nika,michael\_jordan\}@berkeley.edu}}
\date{}

\maketitle

\begin{abstract}
    Reinforcement learning (RL) fine-tuning is widely used in language model training to improve performance on a target task while limiting drift from a reference policy. 
A standard way to balance this trade-off is via a KL-regularized RL objective, although this formulation does not by itself provide a principled way to set the regularization coefficient. In practice, the coefficient is typically chosen heuristically or via hyperparameter search, which can lead to unnecessary overhead in training cost or undesirable reward–retention trade-offs. 
We instead propose a game-theoretic framework that gives this trade-off an explicit statistical interpretation.
Specifically, we study a sequential game in which an agent chooses a policy to maximize cumulative reward while a monitor observes policy outputs over time and tests for deviations from the reference policy. 
Although not originating from the same perspective, we show that the resulting equilibrium policy can nonetheless be expressed as the solution to a KL-regularized RL problem for an optimal regularization parameter that can be viewed as maximizing reward per unit of statistical distinguishability. 
Drawing on classical results from concave-convex fractional programming, we provide a principled method for learning this equilibrium coefficient via reduction to the KL-regularized RL objective, thus allowing for flexible integration into standard fine-tuning pipelines.
In experiments with \qwen and \llama, we show that our methods result in competitive reward-retention trade-offs in a continual learning setting, and illustrate how our framework may be used to audit API providers serving open-source models.\looseness-1
\end{abstract}

\section{Introduction}
A common goal in large language model (LLM) fine-tuning is to improve performance on a target objective while preserving useful behaviors inherited from a reference policy.
These objectives may be at odds with one another, and there are many possible ways to formalize the resulting trade-off. 
One of the most widely used approaches is Kullback--Leibler (KL)-regularized reinforcement learning (RL), which augments the reward objective with a penalty for deviating from the reference policy \citep{jaques2017sequence,jaques2019way,neu2017unified,ziegler2019fine,stiennon2020learning,ouyang2022training}. 
In this formulation, the regularization coefficient
%
%
$\beta \geq 0$ determines how aggressively the fine-tuned policy is penalized for deviating from the reference policy (e.g., a pre-trained LLM). 
When $\beta$ is too small, fine-tuning may increase reward at the cost of significantly changing the model’s behavior. 
When $\beta$ is too large, the model remains close to the reference policy, but may fail to adequately learn the target task. 
In practice, this coefficient is often selected by manual tuning or via grid search~\citep[see, e.g.,][]{ouyang2022training,zhang2023wisdom,lin2024mitigating,tang2024understanding}. 
However this approach can waste compute and lead to undesirable reward–retention trade-offs. 
Some implementations instead adapt $\beta$ online to match a prescribed KL target, but this still requires the learner to specify the target divergence and update heuristic in advance~\citep{schulman2017proximal,ziegler2019fine}.


%


We consider an alternative notion of deviation from the reference policy. 
Rather than measuring deviations directly through a distance or similarity metric, we ask how difficult it is to distinguish the fine-tuned policy from the reference policy based on its outputs. 
This naturally leads to a game between an agent who seeks to maximize reward, and a monitor attempting to detect deviations from the reference policy. 
In this game, an agent deploys a policy to maximize their utility, while a monitor observes policy outputs as they are generated and aims to detect whether those outputs have been generated by the (intended) reference policy or by some other (unintended) one. 
If the monitor detects a deviation from the reference policy, deployment is terminated. 
The agent therefore faces a trade-off between increasing reward and remaining statistically indistinguishable from the reference policy. 
We call this the \emph{sequential detection game}. 
Our key contributions are as follows:\looseness-1



%
\textbf{Game-theoretic formulation for RL fine-tuning (Section~\ref{sec:seq-det-game}).} 
We provide a new perspective on RL fine-tuning through the lens of sequential detection. 
Starting from a game-theoretic formulation, we prove that the agent’s Nash equilibrium strategy in the sequential detection game is to fine-tune the reference policy with the standard KL-regularized RL objective. 
Thus, rather than introducing a new fine-tuning objective, our game-theoretic formulation recovers one that is already used in practice, which allows existing RL fine-tuning algorithms to be applied directly. 
At the same time, it gives the KL penalty an operational interpretation as the statistical cost of becoming distinguishable from the reference policy under sequential monitoring. 
Moreover, the equilibrium identifies the regularization coefficient $\beta^\star$ that optimally trades off between reward maximization and statistical distinguishability. 
Unlike conventional RL fine-tuning, where $\beta$ is chosen heuristically, the equilibrium coefficient is determined solely by the reward function, prompt distribution, and reference policy.\looseness-1

\textbf{Learning the optimal regularization strength (Section~\ref{sec:stoch}).} Building on our equilibrium analysis, we present a stochastic bisection algorithm (Algorithm~\ref{alg:bisect}) that estimates $\beta^\star$ to $\epsilon$-precision by solving $O(\log(1/\epsilon))$ RL sub-problems. 
Algorithm~\ref{alg:bisect} may be interpreted as a game-theoretically principled approach to RL fine-tuning: rather than relying on a separate hyperparameter search, it adaptively computes the equilibrium regularization strength.
The analysis of Algorithm~\ref{alg:bisect} combines classical techniques from the literature on fractional programming \citep{dinkelbach1967nonlinear, schaible1976fractional} and sequential hypothesis testing \citep{robbins1974expected}.\looseness-1 

\textbf{Continual learning and model auditing experiments (Section~\ref{sec:expts}).}
We evaluate our framework in both continual learning and model auditing settings using \qwen~\citep{yang2025qwen3} and \llama \citep{meta2024llama321b}. 
In continual learning, we find preliminary evidence that the equilibrium regularization coefficient yields better reward–retention trade-offs than compute-matched manual tuning. 
In model auditing, we show that the monitor’s equilibrium-prescribed test detects hidden model modifications by open-source API providers more quickly than natural baselines while controlling the type-I error rate.

\section{Preliminaries}\label{sec:background}
We now review the key ingredients from RL fine-tuning, sequential hypothesis testing, and game theory that will later be combined in our analysis.

\textbf{RL fine-tuning.} Consider a prompt space $\cX$ and response space $\cY$. 
A policy $\pi$ is a mapping from prompts to distributions over responses, i.e., $\pi : \cX \rightarrow \Delta(\cY)$. 
We use $\pi(y | x)$ to denote the probability of receiving response $y$ given prompt $x$, and $y \sim \pi(\cdot | x)$ to denote a sample from the conditional distribution over responses given $x$.\footnote{LLMs generate tokens auto-regressively, but for our purposes it suffices to consider only the induced distribution over complete responses.}

Given a reference policy $\pi_{\rf}$, a reward function $r : \cX \times \cY \rightarrow \mathbb{R}$, the KL-regularized RL fine-tuning objective for a regularization parameter $\beta \geq 0$ is defined by
\begin{equation}\label{eq:klrl}
    \max_{\pi}  \left \{ \E_{y \sim \pi(\cdot | x)}[r(x,y)] - \beta \cdot \E \left[ \kl (\pi(\cdot | x), \pi_{\rf}(\cdot | x)) \right] \right \},
\end{equation}
where 
$\kl (\pi(\cdot | x), \pi_{\rf}(\cdot | x)) := \sum_{y \in \cY} \pi(y | x) \log \frac{\pi(y | x)}{\pi_{\rf}(y | x)}$ is the KL divergence between $\pi(\cdot | x)$ and $\pi_{\rf}(\cdot | x)$, and expectations are (also) taken over the prompt distribution $\cD$. 
This optimization admits a closed-form optimal solution \citep{donsker1975variational,rafailov2023direct} of the form\looseness-1 
\begin{equation}\label{eq:tilt}
    \pi_{\beta}(y | x) := \pi_{\rf}(y | x) \frac{\exp(r(x, y)/\beta)}{Z_{\beta}(x)},
\end{equation}
where $Z_{\beta}(x) := \E_{y \sim \pi_{\rf}(\cdot | x)}[\exp(r(x, y)/\beta)]$ is the partition function. We call $\pi_{\beta}$ the tilt of $\pi_{\rf}$. 
Since $Z_{\beta}(x)$ is generally intractable to compute due to the size of $\cY$, RL methods like Proximal Policy Optimization~\citep{schulman2017proximal} and Group Relative Policy Optimization~\citep{shao2024deepseekmath} are typically used to approximately solve Equation \eqref{eq:klrl}.\looseness-1

\textbf{Sequential hypothesis testing}  
is a paradigm of statistical inference for which type-I errors (false positive rates) are controlled not only at pre-determined sample sizes, but also at stopping times \citep{wald1945sequential,wald1947sequential}. Informally, sequential hypothesis tests allow analysts to ``peek'' at their data routinely to adaptively stop experiments and make conclusions for data-dependent reasons.

In the case of simple nulls and simple alternatives, one sees a sequence of observations $z_1, z_2, \ldots$ from a distribution $P_z$ and the goal is to determine whether $P_z = P$ (the null) or $P_z = Q$ (the alternative). 
A sequential hypothesis test, $\phi\equiv (\phi_t)_{t \in \NN}$, is a sequence of functions $\phi_t \equiv \phi(z_1, \dots, z_t)$ outputting either ``reject'' or ``do not reject'' at each time step $t \in \NN$. For a fixed $\alpha \in (0, 1)$, a test is said to be \emph{$\alpha$-correct} if it controls the type-I error at all sample sizes simultaneously; i.e., $P(\exists \; t \in \NN : \phi_t \text{ rejects}) \leq \alpha$. 
Letting $\tau_\alpha \equiv \tau_\alpha(\phi) = \inf \{ t \in \NN : \phi_t \text{ rejects} \}$ be the \emph{stopping time}, notice that $\alpha$-correctness is equivalent to the condition $P(\tau_\alpha < \infty) \leq \alpha$. 
A test is \emph{power-one} if $Q(\tau < \infty) = 1$.\looseness-1



For any $\alpha$-correct, power-one
sequential test, $\mathbb{E}_{Q}[\tau_{\alpha}] \geq \frac{\log(1/\alpha)}{\kl(Q, P)}$~\citep{wald1945sequential}. 
This lower bound is known to be tight in the sense that there exist valid $\alpha$-correct, power-one tests for which 
\begin{equation}\label{eq:order}
    \lim_{\alpha \downarrow 0}\frac{\mathbb{E}_{Q}[\tau_{\alpha}]}{\log(1/\alpha)} = \frac{1}{\kl(Q, P)}.
\end{equation}
A canonical test with this leading-order behavior is the sequential probability ratio test (SPRT) \citep{wald1945sequential} given by $\phi_t = \mathbbm{1} \{ \prod_{i=1}^t \frac{dQ(z_i)}{dP(z_i)} \geq \frac{1}{\alpha}\}$. 

\textbf{Nash equilibria.}
A two-player game consists of strategy spaces $\mathcal{S}_1$, $\mathcal{S}_2$ and utility functions $u_1,u_2 : \mathcal{S}_1 \times \mathcal{S}_2 \rightarrow \mathbb{R}$.
Given a strategy $s_2 \in \mathcal{S}_2$ for player 2, a strategy $s_1 \in \mathcal{S}_1$ is a \emph{best response} for player 1 if $s_1 \in \arg\max_{s_1' \in \mathcal{S}_1} u_1(s_1',s_2).$
Best responses for player 2 are defined analogously. 
A strategy profile $(s_1^*,s_2^*)$ is a \emph{Nash equilibrium} if both strategies are simultaneously best responses, i.e. $u_1(s_1^*,s_2^*) \geq u_1(s_1,s_2^*)$ for all $s_1 \in \mathcal{S}_1$ and $u_2(s_1^*,s_2^*) \geq u_2(s_1^*,s_2)$ for all $s_2 \in \mathcal{S}_2$. 
Thus, at a Nash equilibrium, neither player can improve their utility by unilaterally changing their strategy; i.e., each player's behavior is optimal given the behavior of the other.
\section{The Sequential Detection Game}\label{sec:seq-det-game}
We explore differences between reference policies and target policies from the perspective of a monitor aiming to quickly distinguish one from the other in a two-player (general sum) game against a strategic agent.
%

%

\begin{tcolorbox}[boxrule=0pt, frame empty, breakable]

\textbf{Sequential detection game:}

Consider a setup where a monitor intends to deploy a policy $\pi_\rf$ while an agent potentially deploys a policy $\Tilde{\pi} \neq \pi_\rf$ in place of $\pi_\rf$. The monitor intends to detect such a change from prompt-response outputs while controlling the type-I error rate at a desired level $\alpha \in (0, 1)$.\looseness=-1 \\
Concretely, for each time step $t \in \NN$:
\begin{enumerate}
    \item A prompt $x_t \sim \cD$ is generated from the prompt distribution $\cD$ and a response $y_t \sim \Tilde{\pi}(\cdot \mid x_t)$ is generated from the agent's deployed policy.
    \item The monitor updates a sequential test $\phi_t \equiv \phi((x_1,y_1), \dots, (x_t, y_t))$ subject to the type-I error constraint: $\PP_{y\sim \pi_\rf} \left ( \exists \; t \in \NN : \phi_t \text{ rejects} \right ) \leq \alpha$.
    \item If $\phi_t$ rejects, then the monitor stops deployment, and the agent receives utility $\sum_{i=1}^t r(x_i, y_i)$ according to reward function $r : \cX \times \cY \rightarrow \mathbb{R}$, while the monitor receives utility $-t$.
\end{enumerate}
\end{tcolorbox}
In this sequential detection game, the monitor stops deployment if and only if $\phi_t$ rejects. In other words, their stopping rule is characterized by the stopping time $\tau_\alpha = \inf \{ t\in \NN : \phi_t \text{ rejects} \}$, 
with the convention that $\tau_\alpha = \infty$ if $\phi_t$ never rejects. From this vantage point, the agent wishes to deploy a policy $\Tilde{\pi}$ to maximize their expected cumulative reward before detection, $\E_{y \sim \Tilde{\pi}(\cdot \mid x)}[\sum_{i=1}^{\tau_\alpha} r(x_i, y_i)]$, while the monitor wishes to minimize $\E_{y \sim \Tilde{\pi}(\cdot \mid x)}[\tau_\alpha]$. 

The agent's optimal choice of $\Tilde{\pi}$ turns out to be precisely a KL-regularized tilt of the reference policy, while the monitor's best response turns out to be a likelihood ratio test between $\pi_\rf$ and $\Tilde{\pi}$. To state these facts formally, we first require two assumptions on the agent's reward function $r$.\looseness-1

%
%
%

%
\begin{assumption}[The agent can and must strategize to benefit]\label{ass:pos}
    It is possible for the agent to achieve positive expected utility. That is, there exists a policy $\pi$ such that $\E_{y \sim \pi(\cdot | x)}[r(x, y)] > 0$. Moreover, they must strategize in order to do so, i.e., $\mu_{\rf} := \E_{y \sim \pi_{\rf}(\cdot | x)}[r(x, y)] < 0$.
\end{assumption}

\begin{assumption}[The monitor can always devise a detection strategy]\label{ass:finite-expected-stopping}
    For every $\pi'\neq \pi_\rf$, there exists a test $(\phi_t')_{t\in \NN}$ with a finite expected stopping time, i.e., $\E_{\pi'}[\tau_\alpha(\phi')] < \infty$.
\end{assumption}
Assumption~\ref{ass:pos} provides sufficient conditions for the agent to want to participate in the game and to learn a non-trivial strategy.\footnote{If the latter half of \cref{ass:pos} does not hold, the agent can play $\pi_{\rf}$ and collect infinite reward.} 
For the necessity of strategization in Assumption~\ref{ass:pos}, it is mathematically equivalent to assume that $\mu_{\rf} - c \leq 0$, where $c \geq 0$ is the minimum improvement in expected reward required for the agent to want to update the reference policy. 
The constant $c$ may also be interpreted as capturing the cost associated with model training (e.g., monetary costs, time spent, effort exerted, etc.).
Assumption~\ref{ass:finite-expected-stopping} is an analogous assumption on the monitor's incentive to play the game. 
It is a weak assumption that is satisfied for likelihood ratios in all but pathological cases which we eschew for the purposes of this paper. 

In what follows, we characterize the equilibrium of the sequential detection game. These statements should be interpreted as holding in the so-called ``high-confidence'' regime, i.e., where $\alpha \downarrow 0$.  
This technicality arises because of the discrete-time nature of the sequential tests we consider and does not affect the main interpretation of our result.
Concretely,~\citet{robbins1974expected} prove an information-theoretic lower bound $\log(1/\alpha)/\kl(Q,P)$ for the expected stopping time of any $\alpha$-correct sequential test. 
This bound is tight at the leading $\log(1/\alpha)$ order, but the corresponding upper bounds can include lower-order corrections due to boundary overshoot, e.g., because the log-likelihood ratio is updated at discrete times, it can cross the $\log(1/\alpha)$ boundary and not hit it exactly.
Therefore, the following equilibrium characterization should be interpreted as stating that unilateral deviations cannot improve either player's utility \emph{at the leading $\log(1/\alpha)$ order}.\looseness-1

\textbf{Equilibrium characterization.}
We begin with the monitor's equilibrium strategy. 
Suppose the agent is playing policy $\Tilde{\pi}$. 
Targeting the true agent policy $\Tilde{\pi}$ maximizes the expected evidence accumulated per sample and therefore minimizes the expected stopping time. 
Moreover, it is known that SPRTs exhibit the optimal leading order behavior of Equation~\ref{eq:order}~\citep{siegmund1985sequential}. 
Therefore the monitor’s  best-response to agent policy $\Tilde{\pi}$ is to test for it using an SPRT.

The agent's goal is to pick their policy to maximize $\E \left[\sum_{t=1}^{\tau_{\alpha}} r(x_t, y_t) \right]$, which can be written as $\E[\tau_{\alpha}] \cdot \E_{y \sim \pi(\cdot | x)}[r(x,y)]$ by Wald's equation, since the expected stopping time will be finite under Assumptions~\ref{ass:pos} and~\ref{ass:finite-expected-stopping}. 
By the \citeauthor{robbins1974expected} lower bound, the agent can guarantee themselves utility at least $\log(1/\alpha) \E_{y \sim \pi(\cdot | x)}[r(x,y)] / \E[\kl(\pi(\cdot | x), \pi_{\rf}(\cdot | x))]$, which is also an (asymptotic) upper bound on their utility under the monitor's best-response SPRT. 
Since neither $\E_{y \sim \pi(\cdot | x)}[r(x,y)]$ nor $\E[\kl(\pi(\cdot | x), \pi_{\rf}(\cdot | x))]$ depend on $\alpha$, the agent's optimization takes the form\looseness-1 
%
\begin{equation}\label{eq:frac}
    \max_{\pi} \; \E_{y \sim \pi(\cdot | x)}[r(x,y)] /\E[\kl(\pi(\cdot | x), \pi_{\rf}(\cdot | x))],
\end{equation}
where we use the convention that $0/0 = 0$.
Equation~\ref{eq:frac} is a fractional program, with a concave numerator and a convex, strictly positive denominator (excluding the singular point $\pi=\pi_{\rf}$, which cannot be optimal under Assumption~\ref{ass:pos}). 
As such, we can draw on the rich literature on concave-convex fractional programming~\citep[see, e.g.,][]{dinkelbach1967nonlinear, schaible1976fractional}, which shows that while such fractional optimization problems are not concave, they are quasi-concave and can be solved iteratively through a sequence of concave optimization problems.
\begin{theorem}
    Suppose that for each iteration $k$, the quantities $\pi^{(k)}$ and $\beta^{(k)}$ are chosen by $\pi^{(k)} \in \argmax_{\pi} \left \{ \E_{y \sim \pi(\cdot | x)}[r(x,y)] - \beta^{(k)} \E[\kl(\pi(\cdot | x), \pi_{\rf}(\cdot | x))] \right \}$
    and $\beta^{(k+1)} \leftarrow \frac{\E_{y \sim \pi^{(k)}(\cdot | x)}[r(x,y)]}{ \E[\kl(\pi^{(k)}(\cdot | x), \pi_{\rf}(\cdot | x))]}$. 
    Then it holds that $\beta^{(k)} \to \beta^\star$, where $\pi_{\beta^{\star}} \in \argmax_{\pi} \{ \E_{y \sim \pi(\cdot \mid x)} [r(x,y)] - \beta^\star \E[\kl(\pi(\cdot \mid x), \pi_\rf(\cdot \mid x)] \}$ solves~(\ref{eq:frac}).\looseness-1 
    %
    %
\end{theorem}
Notice that $\pi_{\beta^\star}$ is precisely the optimizer of the RL fine-tuning objective in Equation~\ref{eq:klrl} with regularization coefficient $\beta^\star$. In other words, the optimal policy for the agent is the reference policy $\pi_\rf$ tilted by $\beta^\star$ as in \ref{eq:tilt}.

\textbf{Discussion.} 
Taken together, the results in this section imply that the agent playing policy $\pi_{\beta^\star}$ and the monitor testing for policy $\pi_{\beta^\star}$ is a Nash equilibrium in the sequential detection game.  
The equilibrium characterization yields a game-theoretic derivation for KL-regularized RL, and provides an operational interpretation of the KL penalty as the statistical cost of becoming distinguishable from the reference policy. In particular, the optimal regularization coefficient $\beta^\star$ is the maximum achievable reward per unit of statistical distinguishability, and can be interpreted as the shadow price of distinguishability.

This equilibrium characterization is also robust to tie-breaking. 
In Appendix~\ref{app:tie}, we show that while $\pi_{\beta^\star}$ is not a unique agent best-response, the monitor's SPRT with $\cH_1 : \pi = \pi_{\beta^\star}$ is still an $\alpha$-correct, power-one test for \emph{any} agent best response. 
Therefore, the monitor's test remains valid even if the agent breaks ties in an unknown or arbitrary way. 
\section{Reduction to KL-Regularized RL}\label{sec:stoch}
We now turn our attention from characterizing the equilibrium policy $\pi_{\beta^\star}$ to learning it. 
Traditional fractional programming approaches cannot be applied off-the-shelf, as they rely on the ability to evaluate quantities like $\E_{y \sim \pi(\cdot | x)}[r(x,y)]$ and $\E[\kl(\pi(\cdot | x), \pi_{\rf}(\cdot | x))]$ exactly for an arbitrary policy $\pi$, which is generally not feasible when $\pi$ is an LLM. 
Instead, we show that $\beta^\star$ can be characterized as the unique root of a one-dimensional function and estimate it from data using a stochastic bisection procedure, in the spirit of stochastic root-finding methods \citep{robbins1951stochastic}.

Consider the function $M(\beta) := \E_{y \sim \pi_{\beta}(\cdot | x)}[r(x, y)] - \beta \cdot \E[\kl(\pi_{\beta}(\cdot | x), \pi_{\rf}(\cdot | x))].$
%
%
Intuitively, $M(\beta)$ measures whether the reward that is obtained by deviating from the reference policy exceeds the statistical cost imposed by the KL penalty for the given $\beta$. 
When $\beta$ is too small, the reward dominates and $M(\beta)>0$.
When $\beta$ is too large, the KL penalty dominates and $M(\beta)<0$.
The equilibrium coefficient $\beta^\star$ is precisely the point where these forces balance.
Therefore, computing $\beta^\star$ reduces to finding the unique root of $M(\beta)$. This justification follows from an analysis of \citet{dinkelbach1967nonlinear} (see Lemma~\ref{lem:dink}). 

Since $M(\beta)$ is not available in closed form, we use $\widehat{\E}$ to denote the sample average and instead consider the empirical estimate of $M(\beta)$ using $n \geq 1$ samples, i.e. $\widehat{M}_n(\beta) := \widehat{\E}_{y \sim \pi_{\beta}(\cdot | x)}[r(x, y)] - \beta \cdot \widehat{\E}[\kl(\pi_{\beta}(\cdot | x), \pi_{\rf}(\cdot | x))]$. 
%

%
Our algorithm proceeds as follows: 
Set $\beta_{\lo} = 0$. 
Given an upper bound $\beta_{\hi} \geq \beta^\star$, set $\widehat{\beta} := 0.5(\beta_{\hi} + \beta_{\lo})$ and collect enough samples $n$ such that $\widehat{M}_n(\widehat{\beta}) \pm \rad(\widehat{\beta}, n, \delta)$ is bounded away from zero, where $\rad(\widehat{\beta}, n, \delta)$ is a confidence radius that holds with probability $1-\delta$. 
Depending on the sign of $\widehat{M}_n(\widehat{\beta})$, set either $\beta_{\hi} \leftarrow \widehat{\beta}$ or $\beta_{\lo} \leftarrow \widehat{\beta}$, and repeat this process until $|\beta_{\hi} - \beta_{\lo}|$ is within the desired precision $\epsilon$. 
The full procedure is outlined in Algorithm~\ref{alg:bisect}. 
\begin{algorithm}[t]
\SetAlgoLined
\DontPrintSemicolon
\textbf{Input:} $\epsilon > 0$, $\delta \in (0, 1)$, $\beta_{\hi} \in \mathbb{R}_{+}$, $n_0 \in \mathbb{N}$, $\beta_{\lo} := 0$\\

\While{$\beta_{\hi} - \beta_{\lo} > \epsilon$}{
    $\widehat{\beta} = \frac{1}{2}(\beta_{\hi} + \beta_{\lo})$, $n = n_0$, $\pi_{\widehat{\beta}} \leftarrow \textup{\texttt{RL-Oracle}}(\pi_{\rf}, \widehat{\beta})$\;
    
    \While{$\mathrm{True}$}{
        \uIf{$\widehat{M}_n(\widehat{\beta}) + \rad(\widehat{\beta}, n, \delta) < 0$}{
            $\beta_{\hi} \leftarrow \widehat{\beta}$; break\; 
        }\uElseIf{$\widehat{M}_n(\widehat{\beta}) - \rad(\widehat{\beta}, n, \delta) > 0$}{
            $\beta_{\lo} \leftarrow \widehat{\beta}$; break\;
        }
        \Else{
            $n \leftarrow 2 n$\\
        }
    }
}
\Return $\textup{\texttt{RL-Oracle}}(\pi_{\rf}, \beta_{\lo})$
\caption{Stochastic bisection method}\label{alg:bisect}
\end{algorithm}
\begin{restatable}{theorem}{thmstoch}\label{thm:stoch}
    If $\beta_{\hi} \geq \beta^\star$ and for every $\beta$ queried by Algorithm~\ref{alg:bisect} $\mathbb{P}(\exists \; n \geq 1 : \left|\widehat M_n(\beta)-M(\beta)\right| > \rad (\beta,n,\delta)) \leq \delta$, 
    %
    then with probability at least $1 - \delta \log_2\left(\beta_{\hi} / \epsilon\right)$, Algorithm~\ref{alg:bisect} returns a policy $\pi_{\beta_{\lo}}$ satisfying $0 \leq \beta^\star - \beta_{\lo} \leq \epsilon$ in $\log_2 \left(\beta_{\hi} / \epsilon\right)$ bisection iterations.
\end{restatable}
%
    
%
We instantiate Algorithm~\ref{alg:bisect} with specific choices of $(\beta_{\hi}, \rad(\cdot, \cdot, \cdot))$ in the Appendix (Corollaries~\ref{cor:inst1} and~\ref{cor:inst2}). 
In either instantiation, Algorithm~\ref{alg:bisect} needs $n = O(M(\widehat{\beta})^{-2})$ samples in each iteration to determine the sign of $M(\widehat{\beta})$ with high probability. 
Therefore, as $\widehat{\beta} \rightarrow \beta^\star$ and $|M(\widehat{\beta})| \rightarrow 0$, we should expect the number of samples we need in each iteration of stochastic bisection to increase.
As a practical relaxation, we forego the confidence sequence in our experiments and bisect based on the sign of $\widehat{M}_n(\widehat{\beta})$ after observing a sufficiently large number of samples (e.g., $n=4096$).\looseness-1 

Algorithm~\ref{alg:bisect} posits access to an ``RL oracle'' which, given a reference policy $\pi_{\rf}$ and regularization parameter $\beta$, returns the KL-regularized tilt $\pi_{\beta}$. 
While exact, polynomial-time oracles generally do not exist due to the difficulty of computing the partition function, there is a rich literature of approximate oracles that have been designed to solve the KL-regularized RL problem (e.g., PPO, GRPO, and variants thereof). 
Consequently, our theoretical guarantees for Algorithm~\ref{alg:bisect} should be interpreted as characterizing the complexity of identifying the equilibrium coefficient, assuming that each RL subproblem can be solved sufficiently accurately. 

\textbf{Warm-start version.} 
It is possible to ``warm start'' Algorithm~\ref{alg:bisect}
given an arbitrary setting of $\beta_{\hi}$ as follows: 
Compute $\widehat{M}_n(\beta_{\hi}) + \rad(\beta_{\hi}, n, \delta)$. 
If it is bounded below zero, run Algorithm~\ref{alg:bisect} as normal. 
Otherwise set $\beta_{\lo} \leftarrow \beta_{\hi}$, $\beta_{\hi} \leftarrow 2 \beta_{\hi}$, and repeat. 
This procedure will produce an $\varepsilon$-approximate solution with high probability.  
Furthermore, it makes $\log_2(\beta_{\hi}/\epsilon)$ calls to the RL oracle if $\beta_{\hi}$ is a valid upper bound on $\beta^\star$. 
If it is not, the number of oracle calls is still $O \left(\log_2 \left(\sigma^2/(|\mu_{\rf}|\epsilon) \right)\right)$ if rewards are $\sigma$-sub-Gaussian by Lemma~\ref{lem:high}.\looseness-1

Finally, it is worth noting that in our game formulation, even if the agent and the monitor run separate copies of Algorithm~\ref{alg:bisect}, the agent and monitor utilities will still be close to their equilibrium values. 
See Appendix~\ref{app:app} for more details. 
\section{Experiments}\label{sec:expts}
\begin{figure}[t]
    \centering
    \includegraphics[width=\linewidth]{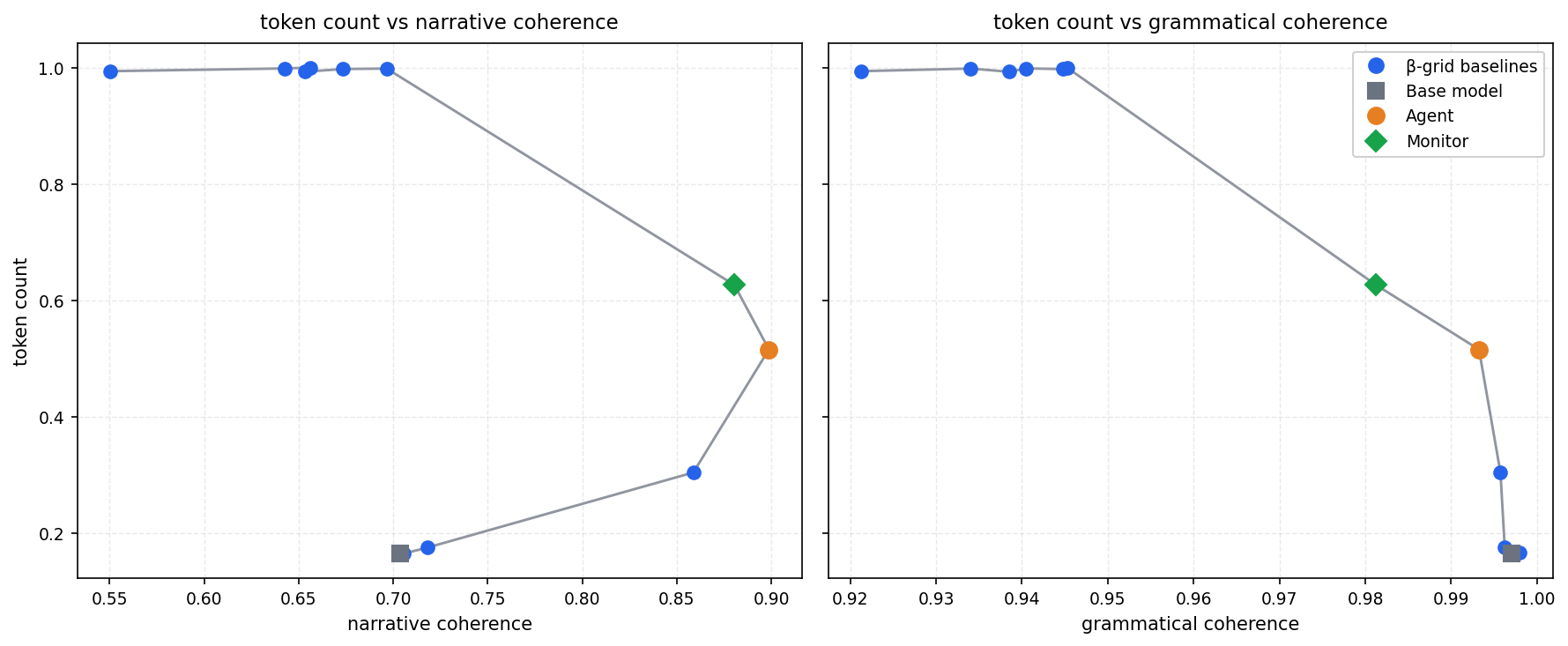}
    \caption{Reward--retention trade-offs for $\qwen$ with shift $\rho = 0.1$. The vertical axis reports response length in tokens, and the horizontal axes report narrative and grammatical coherence, as judged by $\gptfivemini$ \citep{singh2025openai}. Blue circles show policies trained on the compute-matched exponential grid over $\beta$, the gray square denotes the reference policy; and the orange circle and green diamond denote two independent stochastic-bisection runs. Each point is the average of $400$ independent responses. Higher values are better on both axes.\looseness-1}
    \label{fig:qwen}
\end{figure}
We evaluate our methods empirically in two settings: a continual learning task and a model auditing task. 
In the continual learning task, the goal is to improve performance on a new target reward while preserving existing behavioral properties of the reference policy. 
In the model auditing task, an auditor tries to detect whether a model has been strategically fine-tuned from its outputs. 
Taken together, these experiments test the two main operational interpretations of our framework: stochastic bisection as a replacement for manual KL coefficient search, and the equilibrium likelihood-ratio test as a practical auditing rule.

\textbf{Experiment setup.}
The two sets of experiments use the same setup. 
In each round, the language model is given the same prompt (``Once upon a time'') and generates a completion, with the agent receiving reward equal to the number of characters in the response. 

Before bisection, the reward is calibrated by sampling completions from the reference policy and subtracting off the raw mean reward, plus a fixed margin $\rho$. 
Under the reference policy, the calibrated reward therefore has mean approximately equal to $-\rho$. 
The calibrated reward is used throughout KL-constrained fine-tuning and bisection; because adding a constant to the reward does not change the optimal policy at any fixed $\beta$, this shift does not alter the fine-tuning objective. 
Evaluation metrics on Pareto plots are reported in the original, uncalibrated units.\looseness-1

\subsection{Continual Learning}\label{sec:continual}
In our continual learning setting, the response-length reward function plays the role of the skill to be improved, while the reference policy's ability to produce coherent responses is the behavior to be retained. 
Although this is a single model update rather than a multi-task setting, it isolates the central continual learning trade-off: 
aggressively optimizing a new objective can alter or degrade useful behaviors inherited from the reference policy.\footnote{This is commonly observed in the literature on RLHF and is often referred to as the ``alignment tax''~\citep[e.g.,][]{ouyang2022training,lin2024mitigating,zhang2024cppo,jang2024degeneration}.}
For each trained policy, we measure the target behavior using token count and retention using both narrative and grammatical coherence.\footnote{We used \gptfivemini and \gptfivenano as judges to evaluate narrative and grammatical coherence~\citep{zheng2023judging}. All main body plots use \gptfivemini. Results using \gptfivenano are very similar and are in Appendix~\ref{app:expts}.\looseness-1}\looseness-1 

We compare Algorithm~\ref{alg:bisect} against the standard practice of manually sweeping over the regularization coefficient. 
This baseline trains one policy for each value in an exponentially spaced grid over $\beta$, using the same number of fine-tuning runs as Algorithm~\ref{alg:bisect}. 
After inspecting the results, a practitioner could select whichever grid point best matches their desired balance between length and coherence.

Our results are summarized in Figures~\ref{fig:qwen} and~\ref{fig:llama}. 
Higher values are desirable on both axes, so policies toward the upper-right corner of each panel achieve more favorable reward--retention trade-offs.
The points labeled \texttt{Agent} and \texttt{Monitor} are produced by two independent runs of Algorithm~\ref{alg:bisect}. 
We retain these labels for consistency with the auditing experiment in Section~\ref{sec:auditing}; in this experiment, the two points illustrate the variation induced by GRPO RL oracle calls. 
Blue circles correspond to points on an exponential grid between $0$ and $\beta_{\hi}$; the reference policy is denoted by a gray square.\looseness-1 
\begin{figure}[t]
     \centering
     \begin{subfigure}[b]{0.49\textwidth}
         \centering
         \includegraphics[width=\textwidth]{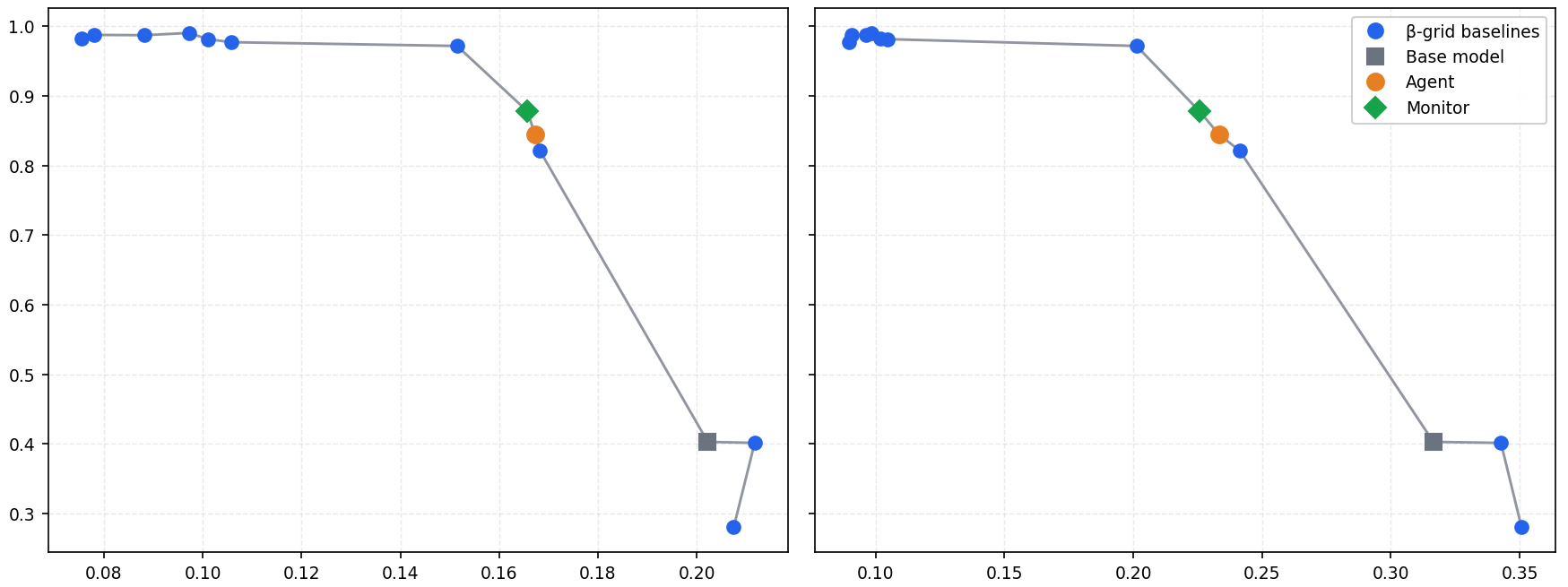}
         \caption{\llama with shift $\rho = 0.1$.}
         \label{fig:llama0p1}
     \end{subfigure}
     \hfill
     \begin{subfigure}[b]{0.49\textwidth}
         \centering
         \includegraphics[width=\textwidth]{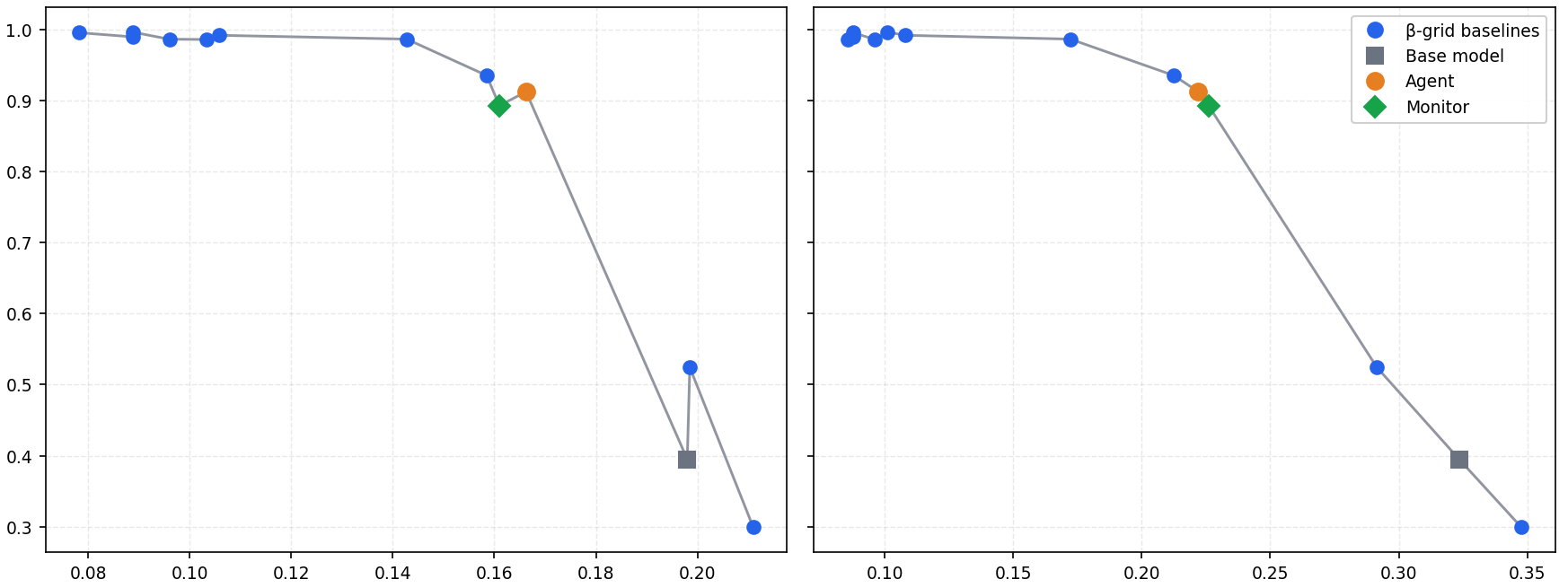}
         \caption{\llama with shift $\rho = 0.2$.}
         \label{fig:llama0p2}
     \end{subfigure}
        \caption{Reward--retention trade-offs for \llama with shift $\rho = 0.1$ (left) and $\rho = 0.2$ (right). Plotting conventions, axis labels, and plot titles are the same as in Figure~\ref{fig:qwen}. In both settings, Algorithm~\ref{alg:bisect} selects policies near the elbow of the empirical trade-off curve traced out by the compute-matched $\beta$ grid.\looseness-1}
        \label{fig:llama}
\end{figure}

Figure~\ref{fig:qwen} shows the results for \qwen with shift $\rho = 0.1$. 
Policies that generate the longest responses exhibit lower coherence, while policies that preserve the highest narrative and grammatical coherence produce shorter responses. 
Interestingly, narrative coherence actually increases with token count up to a point, before decreasing once the token count gets too large. 
Each run of Algorithm~\ref{alg:bisect} selects an intermediate policy near the bends of the empirical trade-off curves, in contrast to the compute-matched grid, which results in models closer to either extreme. 

Figure~\ref{fig:llama} shows the corresponding results for $\llama$ with shifts $\rho = 0.1$ and $\rho = 0.2$. 
The same qualitative pattern holds in both settings:
the bisection policies lie near the ``elbow'' of the curve traced out by the grid, avoiding both the high-retention regime in which little progress is made on the length objective, and the high-length regime in which coherence falls sharply.\footnote{\llama scores considerably lower than \qwen in terms of narrative and grammatical coherence for all values of $\beta$, and does not exhibit the same improvement in narrative coherence as token count increases. This is likely due to it being a much smaller model.}\looseness-1  

Across all settings, we find that Algorithm~\ref{alg:bisect} selects a policy near the elbow of the empirical reward--retention frontier.
In some instances (e.g., Fig.~\ref{fig:qwen}, left), Algorithm~\ref{alg:bisect} finds desirable regions of the frontier that are not reached by grid search. 
However even when Algorithm~\ref{alg:bisect} and grid search find similar parts of the frontier, grid search requires an additional post-hoc selection step, whereas Algorithm~\ref{alg:bisect} resolves this trade-off adaptively by concentrating its fine-tuning runs around $\beta^\star$.\looseness-1

\subsection{Model Auditing}\label{sec:auditing}
\begin{table}[t]
    \centering
    \small
    \setlength{\tabcolsep}{5pt}
    \caption{
    Model-auditing results at level $\alpha=0.05$. $\pi_A$ SPRT uses the true deployed policy and is thus an oracle
    comparator; $\pi_M$ SPRT uses the monitor's independently trained equilibrium policy.
    Lower stopping times and false positive rates are better. Averages are taken over $100$ samples.\looseness-1 
    }
    \label{tab:auditing}
    \begin{tabular}{llccc}
        \toprule
        Model and margin
        & Metric
        & \shortstack{$\pi_A$ SPRT}
        & \shortstack{$\pi_M$ SPRT}
        & \shortstack{Grid\\mixture} \\
        \midrule
        \qwen, $\rho=0.1$
        & Avg. stopping time & 1.00 & 1.08 & 1.02 \\
        & False positive rate & 0.00 & 0.00 & 0.09 \\
        \midrule
        \llama, $\rho=0.01$
        & Avg. stopping time & 2.49 & 6.61 & 20.36 \\
        & False positive rate & 0.04 & 0.06 & 0.20 \\
        \midrule
        \llama, $\rho=0.1$
        & Avg. stopping time & 1.89 & 5.31 & 9.47 \\
        & False positive rate & 0.03 & 0.01 & 0.15 \\
        \midrule
        \llama, $\rho=0.2$
        & Avg. stopping time & 1.75 & 4.40 & 6.23 \\
        & False positive rate & 0.04 & 0.03 & 0.09 \\
        \bottomrule
    \end{tabular}
\end{table}





%
Consider a model-hosting platform that sets inference costs such that they are profitable on a per-token basis. 
Such a platform has an economic incentive to fine-tune its models to produce unnecessarily long responses, while remaining difficult to distinguish from the advertised model~\citep{velasco2026test,cao2026pay}.\looseness-1 

A third-party auditor (e.g., a watchdog organization) can easily observe the model's outputs, but verifying the parameters of the deployed model is significantly more difficult. 
Nevertheless, the auditor may want to detect this strategic behavior when it exists, while controlling the probability of falsely flagging the model when it does not.
Our theory precisely characterizes the equilibrium of this game, and gives the auditor a way to empirically test for this type of bad behavior.\looseness-1 

Starting from the same reference policy and using the token length reward described above, the agent and monitor independently run Algorithm~\ref{alg:bisect} with the same RL oracle budget. 
Let $\pi_{A}$ denote the policy obtained by the agent and let $\pi_{M}$ denote the policy obtained (independently) by the monitor/auditor. 
The agent deploys $\pi_{\mathrm{A}}$, whereas the monitor uses $\pi_{\mathrm{M}}$ as the alternative hypothesis in its sequential test. 
In other words, the auditor knows the reference policy and the incentive being optimized, but they do not know the agent's realized policy or share its fine-tuning randomness. 

Each auditing trial uses a source policy $\pi_{\src} \in \{\pi_{A},\pi_{\rf}\}$. 
Sampling from $\pi_{A}$ represents the presence of strategic fine-tuning, whereas $\pi_{\rf}$ represents honest deployment. 
At time $t$, a completion $y_t \sim \pi_{\src}(\cdot | x)$ is generated and scored under both $\pi_M$ and $\pi_{\rf}$. 
The monitor uses the SPRT which rejects if $L_t > \log(1/\alpha)$ for any $t \geq 1$, where $L_t := \sum_{s=1}^{t} \log(\pi_M(y_s | x)/\pi_{\rf}(y_s | x))$ and the log likelihood is computed by summing the token-level log probabilities along the realized completion.\looseness-1 

We compare the equilibrium test against a uniform mixture likelihood ratio test (described in Appendix~\ref{app:expts}). 
We also compare against the ``skyline'' SPRT which uses the agent's actual deployed policy as its alternative. 
While the agent's policy is not available to the auditor, it gives a lower bound on the detection speed attainable by any SPRT.
Finally, all tests are evaluated on the same sequence of completions from $\pi_{\src}$; the candidate policies are used only to score these observations.\looseness-1

Table~\ref{tab:auditing} summarizes the results. 
For \qwen, the strategic modification is sufficiently pronounced that every test detects it after approximately one observation. 
The agent SPRT stops after exactly one completion on average, while the independently trained monitor stops after $1.08$ completions. 
The uniform mixture test obtains a slightly better stopping time of $1.02$, albeit at the cost of a false positive rate that is higher than the $0.05$ threshold. 
The differences between testing strategies are clearer for \llama. 
As expected, the oracle agent SPRT detects the deployed policy fastest. Nevertheless, the monitor test is significantly faster than the uniform mixture test at every calibration margin.
The composite grid tests also reject more frequently under the reference policy, with empirical false-positive rates between $0.08$ and $0.2$, compared to the monitor's false positive rate which is between $0.01$ and $0.06$.
\section{Related Work}\label{sec:related}

\textbf{Reward–KL trade-offs in RL fine-tuning.}
KL regularization is a central component of modern LLM fine-tuning, particularly in reinforcement learning from human feedback, where it is used to improve a target reward while limiting drift from a reference model. 
Several works study the trade-off between reward and KL regularization empirically \citep[see, e.g.,][]{bai2022training, gao2023scaling}. 
\citet{korbak2022rl} provide a complementary theoretical interpretation of KL-regularized RL through the lens of variational inference, where the reference model is analogous to a prior and the reward function provides evidence. 

\textbf{Game-theoretic applications of hypothesis testing.}
A small but growing line of work uses statistical hypothesis testing procedures as strategic instruments in games. 
\citet{gauthier2026betting} use sequential tests to monitor whether agents in a multi-agent system are behaving consistently with equilibrium play. 
\citet{capitaine2026test} use sequential testing in infinitely repeated games, where they show a ``test-then-punish strategy'' can sustain cooperation by detecting deviations and triggering punishment. 
%
%
\citet{hu2024game} are closer in spirit, as they also study game-theoretic hypothesis testing against strategic evasion. 
They model an attacker who manipulates observations to evade a Neyman--Pearson test and characterize equilibrium distortion strategies and detector responses in Stackelberg and signaling games. 
Our settings differ in both the object being manipulated and the resulting equilibrium structure. 
%

\textbf{Continual learning.}
\citet{sun2019lamol} uses LLMs to generate pseudo-samples of previous tasks for training alongside data for the new task.
\citet{razdaibiedina2023progressive} avoid modifying the reference policy by learning task-specific ``soft prompts'' and concatenating newly learned prompts with earlier ones. 
\citet{zhang2024cppo} study continual learning in RLHF and use sample-wise ``balance weights'' to regulate the trade-off between policy learning and knowledge retention. 
Like us, \citet{zhang2023copf} explicitly consider policies that are tilted from a previous task's policy, but they treat the KL coefficient as a hyperparameter.\looseness-1 
%
%

\textbf{AI auditing.}
Several recent works formalize LLM auditing as a statistical testing problem over model outputs. 
\citet{gao2025model} introduce model equality testing, which asks whether a black-box API is serving the same model as a claimed reference model. 
\citet{richter2025auditing} frame auditing for behavioral shifts in LLMs as a sequential testing problem. 
Our auditing setup is aligned with this statistical view, but differs in that the audited model is itself trained strategically to avoid detection.

\textbf{Strategic learning.}
More broadly, our work belongs to the literature on strategic machine learning, where the deployment of a learning algorithm influences the behavior of other agents.
Strategic classification studies settings in which individuals manipulate their features to receive favorable decisions~\citep{hardt2016strategic}, while performative prediction~\citep{perdomo2020performative} analyzes learning problems where the underlying data distribution changes in response to the deployed model. 
Subsequent work extends these ideas to multi-step settings, where strategic responses can accumulate over time~\citep{harris2021stateful, brown2022performative}, like in our setting.
\section{Conclusions}\label{sec:conc}
We introduced the sequential detection game, a framework for RL fine-tuning in which an agent seeks to maximize reward while remaining difficult to distinguish from a reference policy. 
We showed that the Nash equilibrium recovers the KL-regularized RL objective while determining the regularization coefficient endogenously, giving it an explicit operational interpretation rather than treating it as a tunable hyperparameter. 
We then developed a stochastic bisection algorithm for estimating the equilibrium coefficient and presented empirical evidence that it yields competitive reward–retention trade-offs in a continual learning setting, while also enabling principled auditing of LLM APIs.
Finally, while further work is needed to close the gap between our equilibrium analysis and large-scale RL fine-tuning, we hope that this work will encourage practitioners to use methods whose hyperparameters are derived from principled objectives, rather than chosen heuristically.\looseness-1



\section*{Acknowledgments} 
KH was supported in part by the Simons Institute for the Theory of Computing, and part of this work was conducted when he was visiting the Institute. 
IW-S gratefully acknowledges support from the Miller Institute for Basic Research in Science. 
PA gratefully acknowledges the support of DARPA through award No. HR00112520022.  This work was also funded by the European Union (ERC-2022-SYG-OCEAN-101071601), the NSF Institute for Foundations of Machine Learning under grant CCF-2505865, the National Science Foundation under grant CCF-2145898, the Office of Naval Research under grant N00014-24-1-2159, a Google Research Scholar Award, an Alfred P. Sloan fellowship, and a Schmidt Science AI2050 fellowship.
Views and opinions expressed are however those of the author(s) only and do not
necessarily reflect those of the European Union or the European Research Council
Executive Agency. Neither the European Union nor the granting authority can be
held responsible for them.

\bibliographystyle{plainnat}
\bibliography{refs}

\newpage
\appendix
\section{Appendix for Section~\ref{sec:seq-det-game}: The Sequential Detection Game}

\begin{lemma}[\citet{dinkelbach1967nonlinear}]\label{lem:dink}
    The following properties are true for $M(\beta)$:
    \begin{enumerate}
        \item $M(\beta^\star) = 0$
        \item $M(\beta) < 0$ for all $\beta > \beta^\star$
        \item $M(\beta) > 0$ for all $\beta < \beta^\star$
        \item $M(\beta)$ is convex.
    \end{enumerate}
\end{lemma}

\begin{lemma}\label{lem:R}
    Let $\tau_{\alpha} := \inf\{\tau \; : \; \sum_{t=1}^{\tau} \log \frac{dQ}{dP}(X_t) \geq \log(1/\alpha)\}$ and assume that $Q << P$, $\E_{R}[\tau_{\alpha}] < \infty$, and $\E_{R}[\log \frac{dQ}{dP}(X)] < \infty$. 
    Then 
    \begin{equation*}
        \E_{R}[\tau_{\alpha}] \geq \frac{\log(1/\alpha)}{\E_{R}[\log \frac{dQ}{dP}(X)]}. 
    \end{equation*}
\end{lemma}
\begin{proof}
    By Wald's identity, 
    \begin{equation*}
        \E_{R} \left[\sum_{t=1}^{\tau_{\alpha}} \log \frac{dQ}{dP}(X_t) \right] = \E_R[\tau_{\alpha}] \cdot \E_{R}\left[\log \frac{dQ}{dP}(X) \right].
    \end{equation*}
    Our definition of $\tau_{\alpha}$ implies that 
    \begin{equation*}
        \E_R[\tau_{\alpha}] \cdot \E_{R}\left[\log \frac{dQ}{dP}(X) \right] \geq \log(1/\alpha).
    \end{equation*}
    Dividing both sides by $\E_{R}[\log \frac{dQ}{dP}(X)]$ obtains the result.
\end{proof}

\subsection{Robustness to tie-breaking}\label{app:tie}
While $\pi_{\beta^\star}$ is an optimal policy for the agent to play against the monitor's equilibrium SPRT, it is not the only one. 
\begin{restatable}{theorem}{thmbr}\label{thm:br}
    Let $d(\pi) := \E[\kl(\pi(\cdot | x), \pi_{\rf}(\cdot | x)) - \kl(\pi(\cdot | x), \pi_{\beta^\star}(\cdot | x))]$. 
    Suppose that $\tau_{\alpha} := \inf \left\{\tau \; : \; \sum_{t=1}^{\tau} \log \frac{\pi_{\beta^\star}(x_t, y_t)}{\pi_{\rf}(x_t, y_t)} \geq \log(1/\alpha) \right\}$ and the agent plays policy $\pi$ such that $d(\pi) > 0$. 
    We have that 
    \begin{equation*}
        \max_{\pi} \E_{x_t \sim \cD, y_t \sim \pi(\cdot | x_t)}\left[\sum_{t=1}^{\tau_{\alpha}} r(x_t, y_t) \right] \geq \beta^\star \log(1/\alpha).
    \end{equation*}
\end{restatable}
\begin{proof}
    We have that 
    \begin{equation*}
        \max_{\pi} \E_{x_t \sim \cD, y_t \sim \pi(\cdot | x_t)}\left[\sum_{t=1}^{\tau_{\alpha}} r(x_t, y_t)\right] = \max_{\pi} \E_{x_t \sim \cD, y_t \sim \pi(\cdot | x_t)}[\tau_{\alpha}] \cdot \E_{y \sim \pi(\cdot | x)}[r(x, y)],
    \end{equation*}

    where we know that $d(\pi) > 0$ (and hence $\E_{y \sim \pi(\cdot | x)}[\tau_{\alpha}] < \infty$) by Assumption~\ref{ass:pos}. 
    Applying Lemma~\ref{lem:R}, we have that 
    \begin{equation*}
        \max_{\pi} \E_{x_t \sim \cD, y_t \sim \pi(\cdot | x_t)} \left[\sum_{t=1}^{\tau_{\alpha}} r(x_t, y_t) \right] \geq \max_{\pi} \frac{\E_{y \sim \pi(\cdot | x)}[r(x, y)] \log(1/\alpha)}{\E_{y \sim \pi(\cdot | x)} \left[\log \frac{\pi_{\beta^\star}(y|x)}{\pi_{\rf}(y | x)} \right]}
    \end{equation*}
    The joint distribution induced by $\pi_{\beta^\star}$ is absolutely continuous with respect to $\pi_{\rf}$ since $\pi_{\beta^\star}$ is a tilt of $\pi_{\rf}$.
    Focusing on the denominator, we have that 
    \begin{equation*}
        \E_{y \sim \pi(\cdot | x)} \left[\log \frac{\pi_{\beta^\star}(y|x)}{\pi_{\rf}(y | x)} \right] = \frac{1}{\beta^\star}\E_{y \sim \pi(\cdot | x)}[r(x, y)] - \E[\log Z_{\beta^\star}(x)]
    \end{equation*}
    Setting this aside for a moment, recall that 
    \begin{equation*}
    \begin{aligned}
        M(\beta^\star) &= \E_{y \sim \pi_{\beta^\star}(\cdot | x)}[r(x, y)] - \beta^\star \E[\kl(\pi_{\beta^\star}(\cdot | x), \pi_{\rf}(\cdot | x))]\\
        &= \beta^\star \E[\log Z_{\beta^\star}(x)] = 0
    \end{aligned}
    \end{equation*}
    where the last equality follows from Lemma~\ref{lem:dink}. 
    Since we know that $\beta^\star > 0$ by Assumption~\ref{ass:pos}, we can conclude that $\E[\log Z_{\beta^\star}(x)] = 0$. 
    Substituting this all back in, we can conclude that 
    \begin{equation*}
        \frac{\E_{y \sim \pi(\cdot | x)}[r(x, y)] \log(1/\alpha)}{\E_{y \sim \pi(\cdot | x)}\left[\log \frac{\pi_{\beta^\star}(y|x)}{\pi_{\rf}(y | x)} \right]} = \beta^\star \log(1/\alpha),
    \end{equation*}
    which is independent of the agent's policy $\pi$ as long as $d(\pi) > 0$.
\end{proof}

Therefore, we can conclude that if the monitor is deploying a sequential likelihood ratio test between $\pi_{\beta^\star}$ and $\pi_{\rf}$, the agent can guarantee themselves utility at least $\beta^\star \log(1/\alpha)$ by playing any policy $\pi$ such that $d(\pi) > 0$ (including $\pi = \pi_{\beta^\star}$). 
Moreover as $\alpha \downarrow 0$, this is the optimal utility attainable for the agent under this test.

The test with stopping time $\tau_{\alpha} := \inf\{\tau \; : \; \sum_{t=1}^{\tau} \log \frac{\pi_{\beta^\star}(x_t, y_t)}{\pi_{\rf}(x_t, y_t)} \geq \log(1/\alpha)\}$ is also a valid sequential test for testing the null hypothesis $\cH_0 \; : \; \pi = \pi_{\rf}$ versus the \emph{composite} alternative $\cH_1 \; : \; \pi \text{ is s.t. } d(\pi) > 0$. 

To see this, let $L_t=\log \frac{\pi_{\beta^\star}(y_t|x_t)}{\pi_{\rf}(y_t|x_t)}$ and note that under $\cH_0$, $\exp\left(\sum_{s=1}^t L_s\right)$ is a non-negative martingale with mean $1$. 
By Ville's inequality, we have that 
\begin{equation*}
    \mathbb{P}_{\cH_0}\left(\sup_{t > 0} \exp\left(\sum_{s=1}^t L_s\right) \geq \frac{1}{\alpha} \right) \leq \alpha,
\end{equation*}
which implies that $\mathbb{P}_{\cH_0}(\tau_{\alpha} < \infty) \leq \alpha$. 

Under $\cH_1$, $\E_{y \sim \pi(\cdot | x)}[L_t] = d(\pi) > 0$, so by the law of large numbers we have that $\sum_{s=1}^t L_s \rightarrow \infty$ almost surely, and so $\tau_{\alpha} < \infty$ almost surely. 

Therefore, even if the agent breaks ties among best-response policies in an arbitrary or unknown way, this remains a valid power-1 sequential test for the composite alternative $\cH_1 \; : \; \pi \text{ is s.t. } d(\pi) > 0$. 
%

%






\section{Appendix for Section~\ref{sec:stoch}: Reduction to KL-Regularized RL}
\begin{assumption}[Sub-Gaussian rewards]\label{ass:subG}
    Rewards are sub-Gaussian with parameter $\sigma$, i.e., for any $\pi$ and all $\lambda \in \mathbb{R}$, 
    \begin{equation*}
        \log \E_{y \sim \pi(\cdot | x)} \left[\exp(\lambda(r(x, y)-\mu))\right] \leq \frac{\sigma^2\lambda^2}{2},
    \end{equation*}
    where $\mu := \E_{y \sim \pi(\cdot | x)}[r(x,y)]$.
\end{assumption}

\begin{lemma}[\citet{howard2021time}]\label{lem:sequence}
    Given a sequence of i.i.d.\  observations $(X)_{n=1}^{\infty}$ from a $\sigma$-sub-Gaussian distribution, we have that 
    \begin{equation*}
        \left|\frac{1}{n} \sum_{i=1}^n X_i - \E[X] \right| \leq 1.7 \sigma \sqrt{\frac{0.72\log(10.4/\delta) + \log\log(2n)}{n}}
    \end{equation*}
    with probability at least $1 - \delta$, simultaneously for every $n \geq 1$. 
\end{lemma}

\begin{lemma}\label{lem:high}
    Under Assumption~\ref{ass:subG}, $\beta^\star \leq \frac{\sigma^2}{2|\mu_{\rf}|}$.
\end{lemma}

\begin{proof}
    Consider the Donsker-Varadhan variational formula \citep{donsker1975variational}, which says that for two distributions $P$, $Q$ such that $P \ll Q$ and function $f$ which is measurable,
    \begin{equation*}
        \E_{P}[f] \leq \log \E_Q[\exp(f)] + \kl(P, Q).
    \end{equation*}
    Let $P$ be the joint distribution over $(x, y)$ such that $x \sim \cD$, $y \sim \pi(\cdot | x)$ for some arbitrary policy $\pi$, $Q$ be the joint distribution such that $x \sim \cD$, $y \sim \pi_{\rf}(\cdot | x)$, and $f = \lambda(r(x, y) - \mu_{\rf})$ for some arbitrary $\lambda \in \mathbb{R}$. 
    Plugging this in, we get that
    \begin{equation*}
        \E_{y \sim \pi(\cdot | x)}[r(x, y)] - \mu_{\rf} \leq \frac{1}{\lambda} \left( \log \E_{y \sim \pi_{\rf}(\cdot | x)}[\exp(\lambda(r(x, y) - \mu_{\rf}))] + \E[\kl(\pi(\cdot | x), \pi_{\rf}(\cdot | x))] \right).
    \end{equation*}
    Observe that $P \ll Q$ is implicit in our setting as if this fails to hold, then $\kl(\pi\|\pi_{\mathrm{ref}})=\infty$ so such a policy $\pi$ could never be optimal. 
    Applying Assumption~\ref{ass:subG}, we have that
    \begin{equation*}
    \begin{aligned}
        \E_{y \sim \pi(\cdot | x)}[r(x, y)] - \mu_{\rf} &\leq \frac{1}{\lambda} \left( \frac{\sigma^2 \lambda^2}{2} + \E[\kl(\pi(\cdot | x), \pi_{\rf}(\cdot | x))] \right)\\
        &= \frac{\sigma^2 \lambda}{2} + \frac{1}{\lambda} \E[\kl(\pi(\cdot | x), \pi_{\rf}(\cdot | x))].
    \end{aligned}
    \end{equation*}
    Since $\lambda$ is arbitrary, we can pick it to minimize the right hand side. Applying first-order conditions, we get that the optimal choice of $\lambda$ is 
    \begin{equation*}
        \lambda = \sqrt{\frac{2\E[\kl(\pi(\cdot | x), \pi_{\rf}(\cdot | x))]}{\sigma^2}}.
    \end{equation*}
    Plugging this into our bound, we get that 
    \begin{equation*}
        \E_{y \sim \pi(\cdot | x)}[r(x, y)] - \mu_{\rf} \leq \sqrt{2 \sigma^2 \E[\kl(\pi(\cdot | x), \pi_{\rf}(\cdot | x))]}.
    \end{equation*}
    Rearranging terms and subtracting $\beta \E[\kl(\pi(\cdot | x), \pi_{\rf}(\cdot | x))]$ from each side of the inequality, we get 
    \begin{equation*}
    \begin{aligned}
        \E_{y \sim \pi(\cdot | x)}[r(x, y)] - \beta \E[\kl(\pi(\cdot | x), \pi_{\rf}(\cdot | x))] &\leq \mu_{\rf} + \sqrt{2} \sigma \sqrt{\E[\kl(\pi(\cdot | x), \pi_{\rf}(\cdot | x))]}\\ &- \beta \E[\kl(\pi(\cdot | x), \pi_{\rf}(\cdot | x))].
    \end{aligned}
    \end{equation*}
    Our choice of $\pi$ was arbitrary, so we take the max over $\pi$ on both sides to get
    \begin{equation*}
    \begin{aligned}
        M(\beta) &\leq \mu_{\rf} + \max_{\pi} \left( \sqrt{2} \sigma \sqrt{\E[\kl(\pi(\cdot | x), \pi_{\rf}(\cdot | x))]} - \beta \E[\kl(\pi(\cdot | x), \pi_{\rf}(\cdot | x))] \right)\\
        &= \mu_{\rf} + \max_{k \geq 0} (\sqrt{2} \sigma \sqrt{k} - \beta k).
    \end{aligned}
    \end{equation*}
    Since $\sqrt{2} \sigma \sqrt{k} - \beta k$ is concave in $k$, we can apply the first order condition to compute the maximum, which is obtained at $k = \frac{\sigma^2}{2 \beta^2}$. 
    Substituting this back in, we get that
    \begin{equation*}
        M(\beta) \leq \mu_{\rf} + \frac{\sigma^2}{2 \beta}.
    \end{equation*}
    Using this inequality, we can see that $M(\beta) < 0$ if $\mu_{\rf} + \frac{\sigma^2}{2 \beta} < 0$, or equivalently, if $\beta > \frac{\sigma^2}{-2\mu_{\rf}}$. 
    Finally, taking the contrapositive, we get that if $M(\beta) \geq 0$, then $\beta \leq \frac{\sigma^2}{-2\mu_{\rf}}$. 
    Since $M(\beta^\star) = 0$, we can conclude that $\beta^\star \leq \frac{\sigma^2}{-2\mu_{\rf}} = \frac{\sigma^2}{2|\mu_{\rf}|}$. 
\end{proof}
\thmstoch*
\begin{proof}
    The proof consists of two separate parts:
    \begin{enumerate}
        \item Showing the correctness of Algorithm~\ref{alg:bisect}, conditioned on the clean event
        \item Bounding the number of loop iterations in Algorithm~\ref{alg:bisect} (and therefore the number of high probability events)
    \end{enumerate}
    
    \noindent \textbf{Correctness:} $\beta_{\lo} = 0$ is a valid lower bound on $\beta^\star$, since Assumption~\ref{ass:pos} implies that it is possible for the agent to get non-negative cumulative expected reward. 
    $\beta_{\hi} = \frac{\sigma^2}{2|\mu_{\rf}|}$ is a valid upper bound on $\beta^\star$ by Lemma~\ref{lem:high}. 

    Under the clean event, $M(\widehat{\beta}) < 0$ if $\widehat{M}_n(\widehat{\beta}) + \rad(\widehat{\beta}, n, \delta) < 0$, so we can conclude that $\beta_{\hi} > \widehat{\beta} > \beta^\star$ by parts 1 and 2 of Lemma~\ref{lem:dink}. 
    On the other hand, if $\widehat{M}_n(\widehat{\beta}) - \rad(\widehat{\beta}, n, \delta) > 0$ then $M(\widehat{\beta}) > 0$ under the clean event, so we can conclude that $\beta_{\lo} < \widehat{\beta} < \beta^\star$ by parts 1 and 3 of Lemma~\ref{lem:dink}. 
    Therefore since $M(\beta)$ is a continuous function on the interval $[\beta_{\lo}, \beta_{\hi}]$, under the clean event, $\beta_{\hi} - \beta_{\lo} \leq \epsilon$ and $\beta^\star \in [\beta_{\lo}, \beta_{\hi}]$ once the while loop terminates. 

    \noindent \textbf{Number of iterations:} Our bisection method takes $\log_2(\frac{\beta_{\hi} - \beta_{\lo}}{\epsilon}) = \log_2(\frac{\sigma^2}{2|\mu_{\rf}|\epsilon})$ iterations to ensure that the final parameter returned is within an additive error of $\epsilon$. 
    Assuming a confidence sequence on $\widehat{M}_n(\beta)$ that is valid for all $n \geq 1$ and any fixed $\beta$, our high probability statement requires taking the union bound over each bisection method iteration. 
\end{proof}

\begin{lemma}\label{lem:bounded}
    Under Assumption~\ref{ass:subG}, if $\pi_{\rf}$ is such that $\pi_{\rf}(\iota | h) \geq \gamma$ for any token $\iota \in \cY$ and token prefix $h$, then for any fixed $\beta$,
    \begin{equation*}
        |\widehat{M}_n(\beta) - M(\beta)| \leq 1.7 \left( \sigma + \beta m \log(1/\gamma) \right) \sqrt{\frac{0.72 \log(20.8 / \delta) + \log\log(2n)}{n}},
    \end{equation*}
    with probability at least $1 - \delta$ simultaneously for all $n \geq 1$, where $m$ is the maximum token output sequence length. 
    Therefore, it suffices to set 
    \begin{equation*}
        \rad(\beta, n, \delta) = 1.7 \left( \sigma + \beta m \log(1/\gamma) \right) \sqrt{\frac{0.72 \log(20.8 / \delta) + \log\log(2n)}{n}}
    \end{equation*}
    in Algorithm~\ref{alg:bisect}. 
\end{lemma}
\begin{proof}
    For an arbitrary policy $\pi$ with $x_1, \ldots, x_n \sim \cD$ and $y_i \sim \pi(\cdot | x_i)$, we have that 
    \begin{equation*}
        \left| \frac{1}{n} \sum_{i=1}^n r(x_i, y_i) - \E_{y \sim \pi(\cdot | x)}[r(x, y)] \right| \leq 1.7 \sigma \sqrt{\frac{0.72 \log(10.4 / \delta) + \log\log(2n)}{n}},
    \end{equation*}
    with probability at least $1 - \delta$, simultaneously for every $n \geq 1$ by Lemma~\ref{lem:sequence}. 
    Turning our attention to the KL term, we have that
    \begin{equation*}
    \begin{aligned}
        \E[\kl(\pi(\cdot | x), \pi_{\rf}(\cdot | x))] &= \E_{y \sim \pi(\cdot | x)}\left[\sum_{j=1}^m \log \left( \frac{\pi(y_j | y_{1:j-1}, x)}{\pi_{\rf}(y_j | y_{1:j-1}, x)} \right) \right]\\
        &= \E \left[ \sum_{j=1}^m \E_{y_{1:j-1} \sim \pi(\cdot | x)}\left[\sum_{\iota \in \cY} \pi(\iota | y_{1:j-1}, x) \log \left( \frac{\pi(\iota | y_{1:j-1}, x)}{\pi_{\rf}(\iota | y_{1:j-1}, x)} \right) \right] \right].
    \end{aligned}
    \end{equation*}
    Let $G_j(x) = \E_{y_{1:j-1} \sim \pi(\cdot | x)}\left[\sum_{\iota \in \cY} \pi(\iota | y_{1:j-1}, x) \log \left( \frac{\pi(\iota | y_{1:j-1}, x)}{\pi_{\rf}(\iota | y_{1:j-1}, x)} \right) \right]$, and observe that $0 \leq G_j(x) \leq \log(1/\gamma)$. 

    Consider the one-sample estimator $\widehat{G}_j(x) := \sum_{\iota \in \cY} \pi(\iota | y_{1:j-1}, x) \log \left( \frac{\pi(\iota | y_{1:j-1}, x)}{\pi_{\rf}(\iota | y_{1:j-1}, x)} \right)$ and note that it is unbiased and computable given access to $\pi$, $\pi_{\rf}$, $x$, and $y \sim \pi(\cdot | x)$. 
    We can estimate $\E[\kl(\pi(\cdot | x), \pi_{\rf}(\cdot | x))] = \E[\sum_{j=1}^m G_j(x)]$ as $\frac{1}{n} \sum_{i=1}^n \sum_{j=1}^m \widehat{G}_j(x_i)$. 
    Observing that $\sum_{j=1}^m G_j(x)$ is sub-Gaussian with parameter at most $m \log(1/\gamma)$, we can apply Lemma~\ref{lem:sequence} to get 
    \begin{equation*}
        \left|\frac{1}{n} \sum_{i=1}^n \sum_{j=1}^m \widehat{G}_j(x_i) - \E[\kl(\pi(\cdot | x), \pi_{\rf}(\cdot | x))]\right| \leq 1.7 m \log(1/\gamma) \sqrt{\frac{0.72\log(10.4/\delta) + \log\log(2n)}{n}},
    \end{equation*}
    with probability at least $1 - \delta$, simultaneously for all $n \geq 1$. 
    Combining this with our high-probability guarantee for $\E_{y \sim \pi(\cdot | x)}[r(x, y)]$ and taking a union bound gets us the desired result. 
\end{proof}

\begin{corollary}\label{cor:inst1}
    Under Assumption~\ref{ass:subG}, if $\beta_{\hi} = \frac{\sigma^2}{2|\mu_{\rf}|}$, $\pi_{\rf}(\iota | h) \geq \gamma > 0$ for any token $\iota \in \cY$ and token prefix $h$, and 
    \begin{equation*}
        \rad(\beta, n, \delta) = 1.7 \left( \sigma + \beta m \log(1/\gamma) \right) \sqrt{(0.72 \log(20.8 / \delta) + \log\log(2n))/n},
    \end{equation*}
    then with probability at least $1 - \delta \log_2\left(\frac{\sigma^2}{2|\mu_{\rf}| \epsilon}\right)$, Algorithm~\ref{alg:bisect} returns a policy $\pi_{\beta_{\lo}}$ satisfying $0 \leq \beta^\star - \beta_{\lo} \leq \epsilon$ in $\log_2 \left(\frac{\sigma^2}{2|\mu_{\rf}|\epsilon} \right)$ bisection iterations. 
\end{corollary}
\begin{proof}
    Corollary~\ref{cor:inst1} follows immediately from Theorem~\ref{thm:stoch} and Lemma~\ref{lem:bounded}. 
\end{proof}

\begin{assumption}[Conditional sub-Gaussianity]\label{ass:cond-subG}
    Rewards satisfy Assumption~\ref{ass:subG}. 
    Moreover, under the reference policy, rewards are uniformly conditionally sub-Gaussian, i.e., for every prompt $x \in \cX$ and all $\lambda\in\mathbb{R}$,
    \begin{equation*}
        \log \E_{y\sim \pi_{\rf}(\cdot | x)} \left[ \exp\left(\lambda\bigl(r(x,y)-\mu_{\rm ref}(x)\bigr) \right) \right] \leq \frac{\sigma^2\lambda^2}{2},
    \end{equation*}
    where $\mu_{\rf}(x) := \E_{y\sim \pi_{\rf}(\cdot | x)}[r(x,y)]$.
\end{assumption}

\begin{corollary}\label{cor:inst2}
Under Assumption~\ref{ass:cond-subG}, if
\begin{equation*}
    \rad(\beta,n,\delta) = 1.7 \sigma \sqrt{\left(2+\frac{\sigma^2}{8\beta^2} \right) \frac{0.72\log(10.4/\delta)+\log\log(2n)}{n}}, 
\end{equation*}
and $\beta_{\hi}=\frac{\sigma^2}{2|\mu_{\rf}|}$, then Algorithm~\ref{alg:bisect} returns a policy $\pi_{\beta_{\lo}}$ satisfying $0 \leq \beta^\star - \beta_{\lo} \leq \epsilon$ with probability at least $1-\delta \log_2\left( \frac{\sigma^2}{2|\mu_{\rm ref}|\epsilon} \right)$ in $\log_2 \left(\frac{\sigma^2}{2|\mu_{\rf}|\epsilon} \right)$ bisection iterations. 
\end{corollary}
\begin{proof}
We have that 
\begin{equation*}
\log \frac{\pi_\beta(y | x)}{\pi_{\rf}(y | x)} = \frac{r(x,y)}{\beta}-\log Z_\beta(x).
\end{equation*}
and so for each fixed $x$,
\begin{equation*}
    \kl(\pi_\beta(\cdot | x),\pi_{\rf}(\cdot | x)) = \E_{y \sim \pi_\beta(\cdot | x)}\left[\frac{r(x,y)}{\beta}-\log Z_\beta(x)\right],    
\end{equation*}
which implies that $M(\beta)=\E[\beta\log Z_\beta(x)]$.
By Jensen's inequality and Assumption~\ref{ass:cond-subG}, we have that 
\begin{equation*}
    0 \leq \beta\log Z_\beta(x) - \mathbb{E}_{y\sim\pi_{\rf}(\cdot | x)}[r(x,y)] \leq \frac{\sigma^2}{2\beta}
\end{equation*}
and $\mathbb{E}_{y\sim\pi_{\rf}(\cdot | x)}[r(x,y)]-\mu_{\rf}$ is $\sigma$-sub-Gaussian over $x\sim \cD$. Hence by Cauchy--Schwarz and Hoeffding's lemma,
\begin{equation*}
    \mathbb{E} \exp\left(\lambda\left(\beta\log Z_\beta(x) - M(\beta)\right) \right) \leq \exp\left(\frac{\lambda^2}{2}\left( 2\sigma^2+\frac{\sigma^4}{8\beta^2}\right)\right).
\end{equation*}
Applying Lemma~\ref{lem:sequence} with sub-Gaussian parameter $\sqrt{2\sigma^2+\frac{\sigma^4}{8\beta^2}}$ gives the desired confidence sequence. 
The result now follows from Theorem~\ref{thm:stoch}.
\end{proof}

\subsection{Approximate Equilibria}\label{app:app}

\begin{theorem}
    The following guarantees hold when the monitor and the agent use the tilted policy $\pi_{\beta_{\lo}}$ returned by Algorithm~\ref{alg:bisect}:
    \begin{enumerate}
        \item If the monitor is deploying an SPRT between $\pi_{\beta_{\lo}}$ and $\pi_{\rf}$, the agent can guarantee themselves utility at least $\beta_{\lo} \log(1/\alpha)$ by playing any policy $\pi$ such that $\E_{y \sim \pi(\cdot | x)} \left[\log \frac{\pi_{\beta_{\lo}}(y | x)}{\pi_{\rf}(y | x)} \right] > 0$ (including $\pi = \pi_{\beta_{\lo}}$).\footnote{Recall that $0 \leq \beta^\star - \beta_{\lo} \leq \epsilon$ with high probability, by Theorem~\ref{thm:stoch}.}
        \item If the agent is playing policy $\pi_{\beta_{\lo}}$ and the monitor is playing a power-one simple-vs-simple SPRT, then their best response is to test $\cH_0 \; : \; \pi = \pi_{\rf}$ versus $\cH_1 \; : \; \pi = \pi_{\beta_{\lo}}$. Moreover, as $\alpha \downarrow 0$, this is the optimal such test out of all power-one sequential tests for the monitor to play~\citep{wald1948optimum}. 
    \end{enumerate}
\end{theorem}
\begin{proof}
    \begin{equation*}
    \begin{aligned}
        \max_{\pi} \E_{y_t \sim \pi(\cdot | x_t)} \left[\sum_{t=1}^{\tau_{\alpha}} r(x_t, y_t) \right] &= \max_{\pi} \E_{y_t \sim \pi(\cdot | x)}[\tau_{\alpha}] \cdot \E_{y \sim \pi(\cdot | x)}[r(x, y)]\\
        &\geq \max_{\pi} \frac{E_{y \sim \pi(\cdot | x)}[r(x, y)] \log(1/\alpha)}{\E_{y \sim \pi(\cdot | x)} \left[\log \frac{\pi_{\beta_{\lo}}(y | x)}{\pi_{\rf}(y | x)} \right]}
    \end{aligned}
    \end{equation*}
    where the inequality follows from Lemma~\ref{lem:R} and the fact that $\beta_{\lo} > 0$ with high probability. 
    We can rewrite 
    \begin{equation*}
        \E_{y \sim \pi(\cdot | x)}[\log \frac{\pi_{\beta_{\lo}}(y | x)}{\pi_{\rf}(y | x)}] = \frac{1}{\beta_{\lo}} \E_{y \sim \pi(\cdot | x)}[r(x, y)] - \frac{1}{\beta_{\lo}} M(\beta_{\lo}).
    \end{equation*}
    Since $\beta_{\lo} \leq \beta^\star$, we know that $M(\beta_{\lo}) > 0$, and so 
    \begin{equation*}
        \max_{\pi} \frac{E_{y \sim \pi(\cdot | x)}[r(x, y)] \log(1/\alpha)}{\E_{y \sim \pi(\cdot | x)}[\log \frac{\pi_{\beta_{\lo}}(y | x)}{\pi_{\rf}(y | x)}]} \geq \beta_{\lo} \log(1/\alpha)
    \end{equation*}
\end{proof}

Observe that even if each agent runs Algorithm~\ref{alg:bisect} separately, if they stop after the same number of iterations they will arrive at the same value for $\beta_{\lo}$ with high probability.
With that being said, it is still possible to say something about the quality of the equilibrium whenever the monitor and the agent run Algorithm~\ref{alg:bisect} for a different number of iterations. 
\begin{theorem}
    Let $\beta_{A}$ (resp. $\beta_{M}$) be the agent's (resp. monitor's) computation of $\beta_{\lo}$ and suppose that $\beta_A \leq \beta_M$, i.e., the clean event holds and the monitor runs Algorithm~\ref{alg:bisect} for at least as long as the agent. 
    Then:
    \begin{enumerate}
        \item If the monitor is deploying a SPRT between $\pi_{\beta_{M}}$ and $\pi_{\rf}$, the agent can guarantee themselves utility at least $(\beta^\star - \epsilon) \log(1/\alpha)$ by playing policy $\pi_{\beta_A}$. 
        \item If the agent is playing policy $\pi_{\beta_{A}}$ and the monitor plays a power-one simple SPRT between $\cH_0 : \pi =  \pi_{\rf}$ and $\cH_1 : \pi = \pi_{\beta_M}$, then as $\alpha \downarrow 0$
        \begin{equation*}
            \tau^* \leq \E_{y \sim \pi_{\beta_A}}[\tau_{\alpha}] \leq \frac{\beta_M}{\beta_A} \cdot \tau^*,
        \end{equation*}
        where $\tau^*$ is the best possible expected stopping time. 
    \end{enumerate}
\end{theorem}
\begin{proof}
    Part 1: The key step is to show that $\E_{y \sim \pi_A(\cdot | x)}[\log \frac{\pi_M(y | x)}{\pi_{\rf}(y | x)}] > 0$. 
    To see this, observe that 
    \begin{equation*}
        \E_{y \sim \pi_{\beta_A}(\cdot | x)}\left[\log \frac{\pi_{\beta_M}(y | x)}{\pi_{\rf}(y | x)} \right] = \frac{1}{\beta_M} (\E_{y \sim \pi_{\beta_A}(\cdot | x)}[r(x, y)] - M(\beta_M))
    \end{equation*}
    and 
    \begin{equation*}
        \E_{y \sim \pi_{\beta_A}(\cdot | x)}[r(x, y)] =  M(\beta_A) + \beta_A \E[\kl(\pi_{\beta_A}(\cdot | x), \pi_{\rf}(\cdot | x))] > M(\beta_A) \geq M(\beta_M)
    \end{equation*}
    where the last inequality follows from the fact that $\beta_A \leq \beta_M$. 
    Given this, we can bound 
    \begin{equation*}
    \begin{aligned}
        \E_{y_t \sim \pi_{\beta_A}(\cdot | x_t)} \left[\sum_{t=1}^{\tau_{\alpha}} r(x_t, y_t) \right] &= \E_{y_t \sim \pi_{\beta_A}(\cdot | x)}[\tau_{\alpha}] \cdot \E_{y \sim \pi_{\beta_A}(\cdot | x)}[r(x, y)]\\
        &\geq \frac{E_{y \sim \pi_{\beta_A}(\cdot | x)}[r(x, y)] \log(1/\alpha)}{\E_{y \sim \pi_{\beta_A}(\cdot | x)} \left[\log \frac{\pi_{\beta_{M}}(y | x)}{\pi_{\rf}(y | x)} \right]}\\
        &= \frac{E_{y \sim \pi_{\beta_A}(\cdot | x)}[r(x, y)] \log(1/\alpha)}{\frac{1}{\beta_{M}} \E_{y \sim \pi_{\beta_A}(\cdot | x)}[r(x, y)] - \frac{1}{\beta_{M}} M(\beta_{M})}\\
        &\geq \beta_M \log(1/\alpha) \geq (\beta^\star - \epsilon) \log(1/\alpha)
    \end{aligned}
    \end{equation*}
    
    Part 2: 
    Let $\tau_{\alpha}(\beta_M)$ be the stopping time when the monitor is using a simple versus simple SPRT with alternative hypothesis $\pi = \pi_{\beta_M}$, and $\tau_{\alpha}(\beta_A)$ be the counterfactual stopping time when using alternative hypothesis $\pi = \pi_{\beta_A}$. 
    Consider the limit where $\alpha \downarrow 0$. 
    We have that 
    \begin{equation*}
    \begin{aligned}
        \E_{y \sim \pi_{\beta_A}(\cdot | x)} \left[\log \frac{\pi_{\beta_M}(\cdot | x)}{\pi_{\rf}(\cdot | x)}\right]
        &= \frac{1}{\beta_M} \E_{y \sim \pi_{\beta_A}(\cdot | x)}[r(x, y)] - \frac{1}{\beta_M} M(\beta_M)\\
        &\geq \frac{1}{\beta_M} \E_{y \sim \pi_{\beta_A}(\cdot | x)}[r(x, y)] - \frac{1}{\beta_M} M(\beta_A)\\
        &= \frac{\beta_A}{\beta_M} \left( \frac{1}{\beta_A} \E_{y \sim \pi_{\beta_A}(\cdot | x)}[r(x, y)] - \frac{1}{\beta_A} M(\beta_A) \right)
    \end{aligned}
    \end{equation*}
    Therefore, 
    \begin{equation*}
    \begin{aligned}
        \frac{\E_{y \sim \pi_{\beta_A}}[\tau_{\alpha}]}{\tau^*} &= \frac{\E_{y \sim \pi_{\beta_A}}[\tau_{\alpha}(\beta_M)]}{\E_{y \sim \pi_{\beta_A}}[\tau_{\alpha}(\beta_A)]}\\
        &= \frac{\E_{y \sim \pi_{\beta_A}(\cdot | x)} \left[\log \frac{\pi_{\beta_A}(y | x)}{\pi_{\rf}(y | x)} \right]}{\E_{y \sim \pi_{\beta_A}(\cdot | x)} \left[\log \frac{\pi_{\beta_M}(y | x)}{\pi_{\rf}(y | x)} \right]}
        \leq \frac{\beta_M}{\beta_A}.
    \end{aligned}
    \end{equation*}
\end{proof}

\section{Appendix for Section~\ref{sec:expts}: Experiments}\label{app:expts}
Policies are fine-tuned with rank 16 LoRA adapters~\citep{hu2022lora} using GRPO~\citep{shao2024deepseekmath} as the RL oracle and Adam~\citep{kingma2014adam} as the optimizer with learning rate 1e-4, temperature $1.0$, top-$p$ $1.0$, and a completion cap of 2048 tokens. 
Each call to the RL oracle consists of 15 optimization steps with 400 rollouts each. 
All models were trained using the Tinker API~\citep{thinkingmachines2025tinker}.

\subsection{Continual Learning}\label{app:cl}
\paragraph{Judge prompts.}
For each trained policy, we evaluated generated completions using two separate LLM-judge prompts: one for narrative coherence and one for grammatical coherence. The same prompts were used for both \gptfivemini and \gptfivenano.

\noindent\textbf{Narrative coherence system prompt:}
\begin{quote}
    You are an expert literary critic evaluating short stories written by a small language model. Your task is to rate the NARRATIVE COHERENCE of a story on an integer scale from 0 to 10.
    
    Narrative coherence means:
    
    The story has a recognizable beginning, middle, and end.
    Events follow each other in a sensible causal or temporal order.
    Characters and settings remain consistent throughout.
    The story arrives at some kind of resolution or conclusion.
    Ignore grammar and spelling mistakes -- those are evaluated separately. Focus only on the structure and coherence of the narrative.
    
    Respond with ONLY a single integer between 0 and 10, with no other text.
    0 = no discernible narrative; 10 = a complete, well-structured short story.
\end{quote}
\noindent\textbf{Grammatical coherence system prompt:}
\begin{quote}
    You are an expert linguist evaluating short stories written by a small language model. Your task is to rate the GRAMMATICAL COHERENCE of a story on an integer scale from 0 to 10.
    
    Grammatical coherence means:
    
    Sentences are syntactically well-formed.
    Subject-verb agreement, tense, pronouns, and articles are used correctly.
    Punctuation and capitalization roughly follow standard English conventions.
    Words are spelled correctly and used in plausible contexts.
    Ignore plot quality and narrative structure -- those are evaluated separately. Focus only on syntactic and morphological correctness.
    
    Respond with ONLY a single integer between 0 and 10, with no other text.
    0 = essentially ungrammatical; 10 = fully grammatical, idiomatic English.
\end{quote}

\begin{figure}
    \centering
    \includegraphics[width=\linewidth]{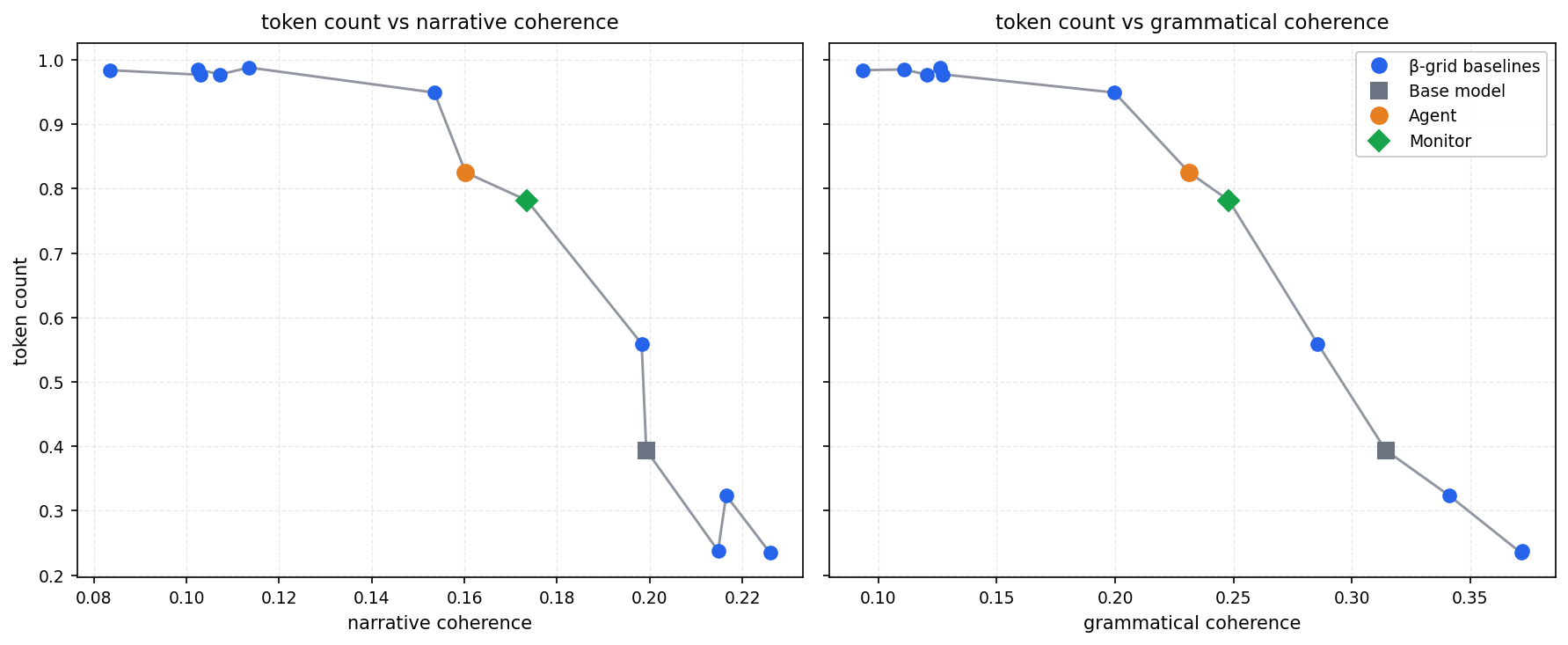}
    \caption{Reward--retention trade-offs for \llama with shift $\rho = 0.01$. Plotting conventions are the same as in Figure~\ref{fig:qwen}. In both settings, Algorithm~\ref{alg:bisect} selects policies near the elbow of the empirical trade-off curve traced out by the compute-matched $\beta$ grid.}
    \label{fig:llama_0p01}
\end{figure}

Figure~\ref{fig:llama_0p01} reports the additional \llama experiment with calibration margin $\rho=0.01$, using \gptfivemini as the judge.

\begin{figure}
    \centering
    \includegraphics[width=\linewidth]{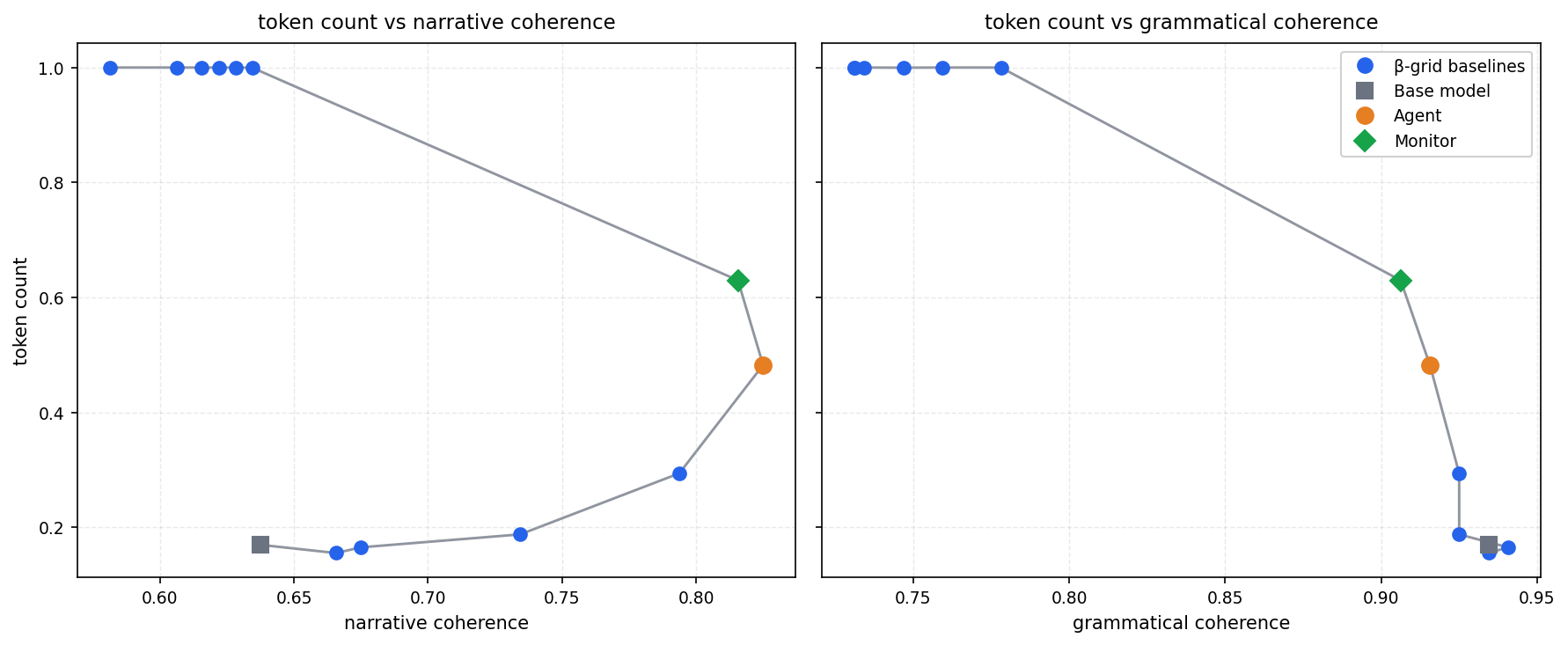}
    \caption{Reward--retention trade-offs for $\qwen$ with shift $\rho = 0.1$ and \gptfivenano-as-a-judge. All other plotting conventions are the same as in Figure~\ref{fig:qwen}.\looseness-1}
    \label{fig:qwennano}
\end{figure}

\begin{figure}
    \centering
    \includegraphics[width=\linewidth]{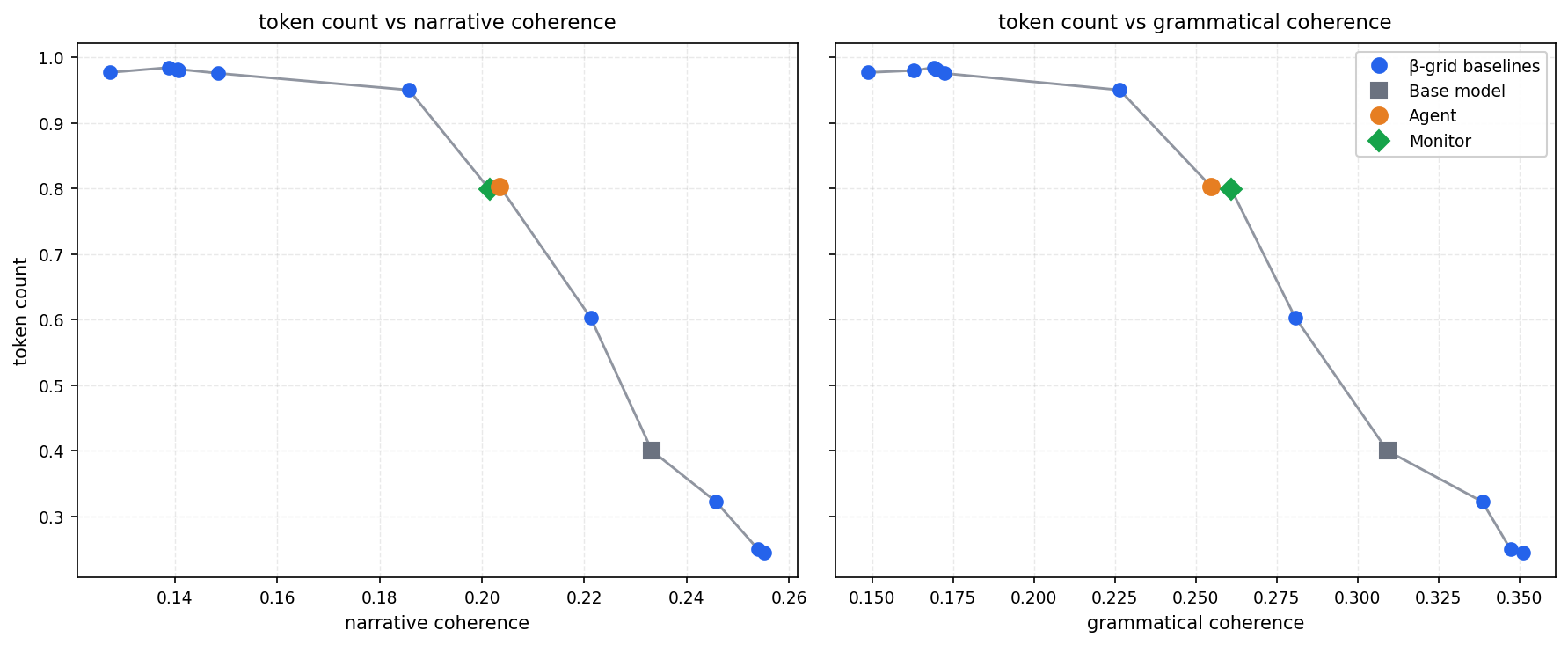}
    \caption{Reward--retention trade-offs for \llama with shift $\rho = 0.01$ and \gptfivenano-as-a-judge.}
    \label{fig:llama_0p01_nano}
\end{figure}

\begin{figure}
    \centering
    \includegraphics[width=\linewidth]{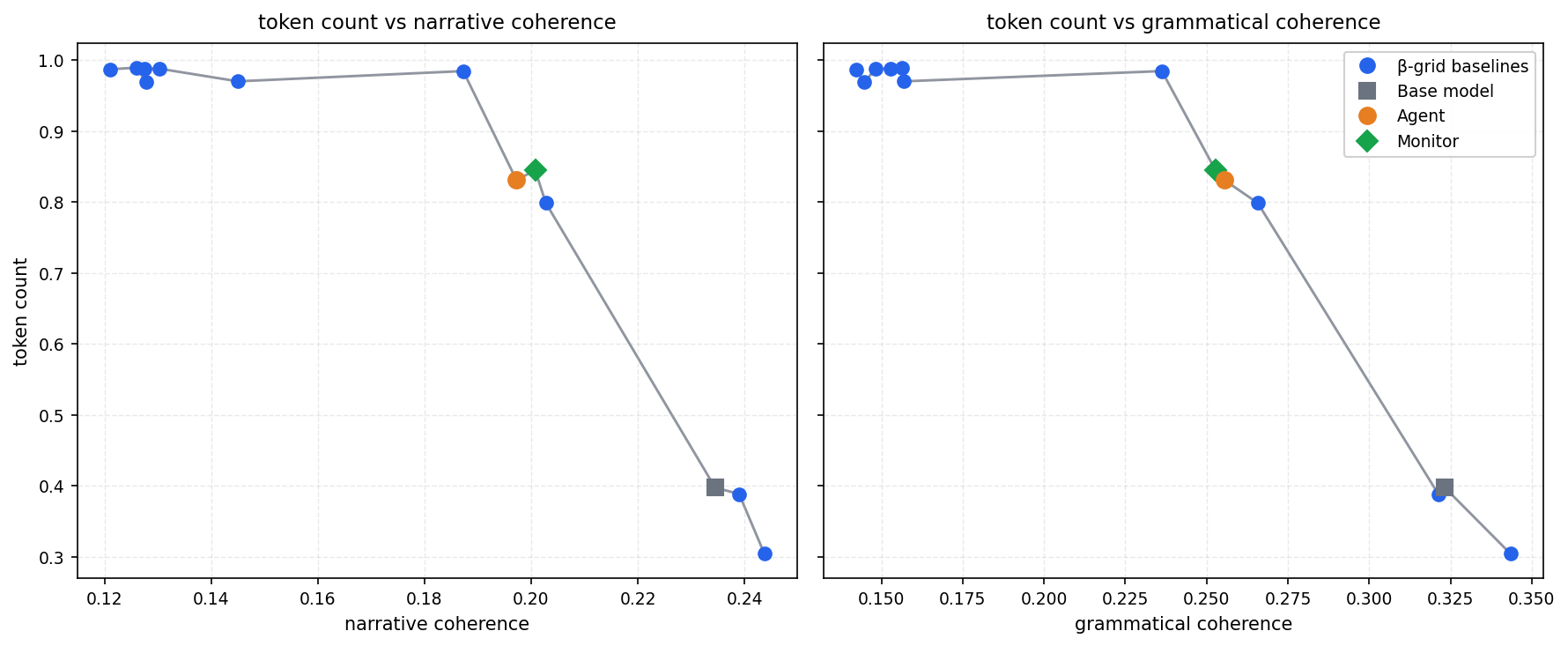}
    \caption{Reward--retention trade-offs for \llama with shift $\rho = 0.1$ and \gptfivenano-as-a-judge.}
    \label{fig:llama_0p1_nano}
\end{figure}

\begin{figure}
    \centering
    \includegraphics[width=\linewidth]{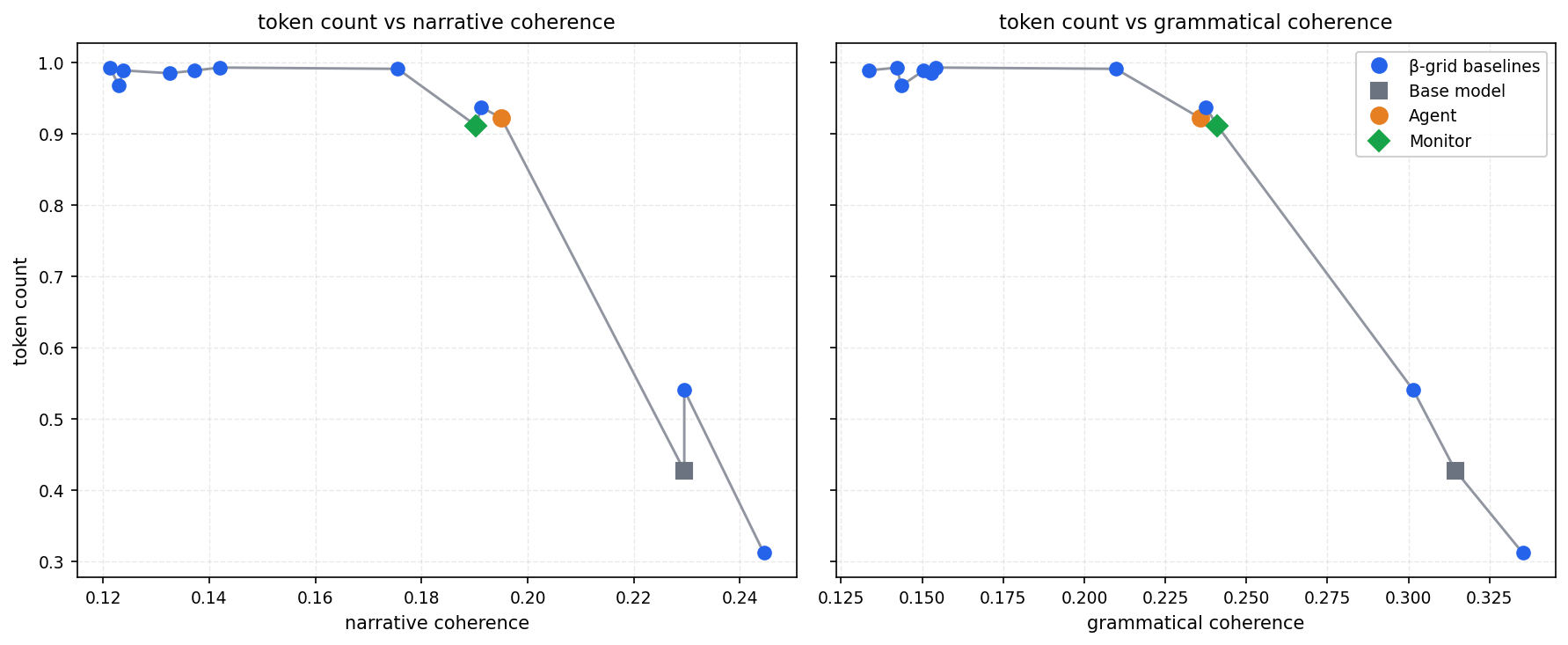}
    \caption{Reward--retention trade-offs for \llama with shift $\rho = 0.2$ and \gptfivenano-as-a-judge.}
    \label{fig:llama_0p2_nano}
\end{figure}

Figures~\ref{fig:qwennano}--\ref{fig:llama_0p2_nano} repeat the continual-learning experiments using \gptfivenano as the judge. 
The absolute coherence scores differ slightly from the \gptfivemini scores, but the qualitative pattern is unchanged: 
Across \qwen and \llama, and across the tested calibration margins, the stochastic-bisection policies fall in the transition region of the Pareto frontier rather than at either extreme of the compute-matched $\beta$ grid.

\paragraph{Plot details.}
This subsection contains the $\beta$ values that are used in all figures, and the order in which they are plotted. 
For \qwen, all $\beta$ values are listed in clockwise order.
For \llama, all $\beta$ values are listed from left to right. 

\noindent \textbf{\qwen with shift $\rho = 0.1$ details:}
\begin{itemize}
    \item Figure~\ref{fig:qwen}, left: $0.00058$, $2.693e-05$, $0.000125$, $0.00269$, $0.0$, $5.802e-06$, $0.008545$ (monitor), $0.008545$ (agent), $0.0125$, $0.0580$, $0.2693$, $1.25$, base model.
    \item Figure~\ref{fig:qwen}, right: $0.00058$, $5.802e-06$, $0.00269$, $2.693e-05$, $0.0$, $0.000125$, $0.008545$ (monitor), $0.008545$ (agent), $0.0125$, $0.0580$, $0.2693$, $1.25$, base model.
    \item Figure~\ref{fig:qwennano}, left: $0.00058$, $0.0$, $2.693e-05$, $0.00269$, $5.802e-06$, $0.000125$, $0.008545$ (monitor), $0.008545$ (agent), $0.0125$, $0.0580$, $1.25$, $0.2693$, base model.
    \item Figure~\ref{fig:qwennano}, right: $0.0$, $2.693e-05$, $0.00058$, $0.000125$, $0.00269$, $5.802e-06$, $0.008545$ (monitor), $0.008545$ (agent), $0.0125$, $0.0580$, $1.25$, $0.2693$, base model.
\end{itemize}

\noindent \textbf{\llama with shift $\rho = 0.01$ details:}
\begin{itemize}
    \item Figure~\ref{fig:llama_0p01}, left: $0.0004$, $0.00884$, $0.0$, $0.0019$, $8.8437e-05$, $0.041$, $0.093$ (agent), $0.093$ (monitor), $0.19$, base model, $4.10$, $0.884$, $19.05$
    \item Figure~\ref{fig:llama_0p01}, right: $0.0004$, $0.0$, $0.00884$, $8.8437e-05$, $0.0019$, $0.041$, $0.093$ (agent), $0.093$ (monitor), $0.19$, base model, $0.884$, $19.05$, $4.10$
    \item Figure~\ref{fig:llama_0p01_nano}, left:     $0.0004$, $0.0019$, $0.0$, $0.00884$, $8.8437e-05$, $0.041$, $0.093$ (monitor), $0.093$ (agent), $0.19$, base model, $0.884$, $19.05$, $4.10$
    \item Figure~\ref{fig:llama_0p01_nano}, right:     $0.0004$, $0.0$, $0.0019$, $0.00884$, $8.8437e-05$, $0.041$, $0.093$ (agent), $0.093$ (monitor), $0.19$, base model, $0.884$, $19.05$, $4.10$
\end{itemize}

\noindent \textbf{\llama with shift $\rho = 0.1$ details:}
\begin{itemize}
    \item Figure~\ref{fig:llama0p1}, left: $0.0$, $0.0042$, $4.202e-05$, $0.000195$, $0.000905$, $9.0536e-06$, $0.0195$, $0.08381$ (monitor), $0.08381$ (agent), base model, $0.0905$, $0.4202$, $1.95$
    \item Figure~\ref{fig:llama0p1}, right: $9.0536e-06$, $0.0042$, $4.202e-05$, $0.000195$, $0.0$, $0.000905$, $0.0195$, $0.08381$ (monitor), $0.08381$ (agent), $0.0905$, base model, $0.4202$, $1.95$
    \item Figure~\ref{fig:llama_0p1_nano}, left:     $0.0042$, $0.0$, $4.202e-05$, $0.000195$, $9.0536e-06$, $0.000905$, $0.0195$, $0.08381$ (agent), $0.08381$ (monitor), $0.0905$, base model, $0.4202$, $1.95$
    \item Figure~\ref{fig:llama_0p1_nano}, right:     $0.0042$, $4.202e-05$, $9.0536e-06$, $0.000195$, $0.0$, $0.000905$, $0.0195$, $0.08381$ (monitor), $0.08381$ (agent), $0.0905$, $0.4202$, base model, $1.95$
\end{itemize}

\noindent \textbf{\llama with shift $\rho = 0.2$ details:}
\begin{itemize}
    \item Figure~\ref{fig:llama0p2}, left: $9.352e-05$, $0.002$, $0.000434$, $4.34e-06$, $0.0$, $2.0148e-05$, $0.00935$, $0.0434$, $0.0603$ (monitor), $0.0603$ (agent), base model, $0.201$, $0.935$
    \item Figure~\ref{fig:llama0p2}, right: $4.34e-06$, $0.002$, $9.352e-05$, $0.0$, $0.000434$, $2.0148e-05$, $0.00935$, $0.0434$, $0.0603$ (agent), $0.0603$ (monitor), $0.201$, base model, $0.935$
    \item Figure~\ref{fig:llama_0p2_nano}, left: $9.352e-05$, $0.002$, $4.34e-06$, $0.000434$, $0.0$, $2.0148e-05$, $0.00935$, $0.0603$ (monitor), $0.0434$, $0.0603$ (agent), base model, $0.201$, $0.935$
    \item Figure~\ref{fig:llama_0p2_nano}, right: $4.34e-06$, $9.352e-05$, $0.002$, $0.0$, $0.000434$, $2.0148e-05$, $0.00935$, $0.0603$ (agent), $0.0434$, $0.0603$ (monitor), $0.201$, base model, $0.935$
\end{itemize}

\subsection{Model Auditing}
As a baseline, we fine-tune surrogate policies on an exponential $\beta$-grid (the same one as in Appendix~\ref{app:cl}) and, at each timestep, compute per-arm log-likelihood ratios against the reference on the same observed pairs. 
The mixture likelihood-ratio test aggregates evidence through a uniform mixture over grid points, stopping when the mixture statistic $\Lambda_t=\sum_k w_k \exp(L_{k,t})$ exceeds $1/\alpha$, where $L_{k,t}$ is the cumulative log-LR for arm $k$ and $w_k=1/K$. 
Under strategic sampling we compare the stopping times of both composites to those of the monitor and the bisection-optimal agent; under honest sampling we report the fraction of trials on which each composite crosses its threshold.

\end{document}